\documentclass{article}

\PassOptionsToPackage{numbers, compress}{natbib}

\usepackage[preprint]{neurips_2022}

\usepackage[utf8]{inputenc} 
\usepackage[T1]{fontenc}    
\usepackage{hyperref}       
\usepackage{url}            
\usepackage{booktabs}       
\usepackage{amsfonts}       
\usepackage{nicefrac}       
\usepackage{microtype}      
\usepackage{xcolor}         

\usepackage{subfig}
\usepackage{graphicx}
\usepackage{bm}
\usepackage{amsthm}

\newtheorem{lemma}{Lemma}

\usepackage{amsmath}
\usepackage{amssymb}
\usepackage{multirow}
\usepackage[font=small, skip=1pt]{caption}
\usepackage{listings}

\definecolor{codegreen}{rgb}{0,0.6,0}
\definecolor{codegray}{rgb}{0.5,0.5,0.5}
\definecolor{codepurple}{rgb}{0.58,0,0.82}
\definecolor{backcolour}{rgb}{0.95,0.95,0.92}

\lstdefinestyle{mystyle}{
  commentstyle=\color{codegreen},
  keywordstyle=\color{black},
  numberstyle=\tiny\color{codegray},
  basicstyle=\ttfamily\footnotesizea,
  breakatwhitespace=false,
  breaklines=true,
  keepspaces=true,
  numbers=left,
  numbersep=5pt,
  showspaces=false,
  showstringspaces=false,
  showtabs=false,
  tabsize=2
}
\def\eg{e.g.} 
\def\ie{i.e.} 
\usepackage[capitalize]{cleveref}
\crefname{section}{Sec.}{Secs.}
\Crefname{section}{Section}{Sections}
\Crefname{table}{Table}{Tables}
\crefname{table}{Tab.}{Tabs.}

\title{Towards Privacy-Preserving, Real-Time and Lossless Feature Matching}

\author{
  Qiang Meng \\
  Algorithm Research, Aibee Inc. \\
  ustcmq@gmail.com
  \And 
  Feng Zhou \\
  Algorithm Research, Aibee Inc. \\
  fzhou@aibee.com
}
\begin{document}

\maketitle

\begin{abstract}
  Most visual retrieval applications store feature vectors for downstream matching tasks.
  These vectors, from where user information can be spied out, will cause privacy leakage if not carefully protected.
  To mitigate privacy risks, current works primarily utilize non-invertible transformations or fully cryptographic algorithms.
  However, transformation-based methods usually fail to achieve satisfying matching performances while cryptosystems suffer from heavy computational overheads.
  In addition, secure levels of current methods should be improved to confront potential adversary attacks.
  To address these issues, this paper proposes a plug-in module called SecureVector that protects features by random permutations, 4L-DEC converting and existing homomorphic encryption techniques.
  For the first time, SecureVector achieves real-time and lossless feature matching among sanitized features, along with much higher security levels than current state-of-the-arts.
  Extensive experiments on face recognition, person re-identification, image retrieval, and privacy analyses demonstrate the effectiveness of our method.
  Given limited public projects in this field, codes of our method and implemented baselines are made open-source in \url{https://github.com/IrvingMeng/SecureVector}.
\end{abstract}

\section{Introduction}
\begin{figure}[htb!]
  \centering
  \subfloat[
  Overview of the proposed SecureVector.\label{fig:intro}
  ]
  {
    \includegraphics[trim={0 0 0 7pt},clip, width=0.45\textwidth]{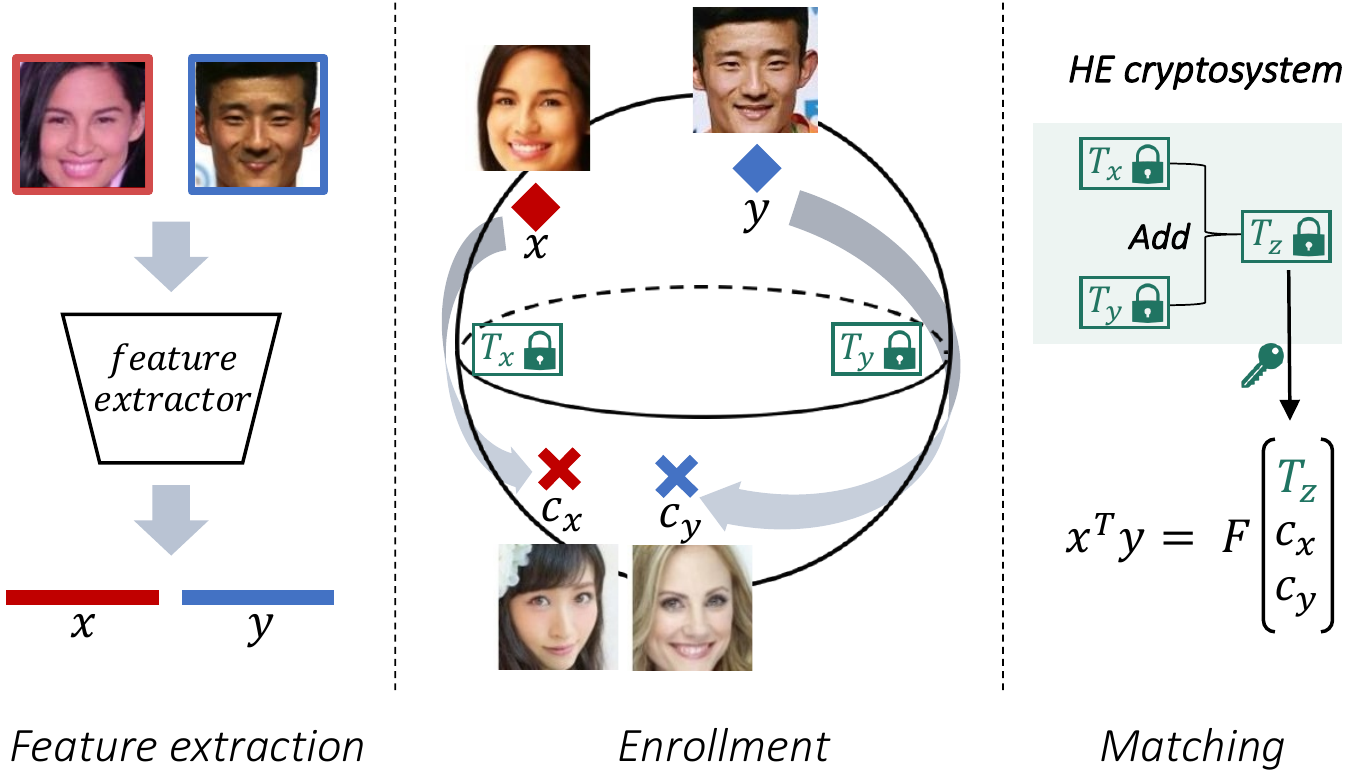}      
  }
  \quad 
  \subfloat[
  Performance-speed analysis of plug-in methods.\label{fig:intro0}
  ]
  {
    \includegraphics[width=0.48\textwidth]{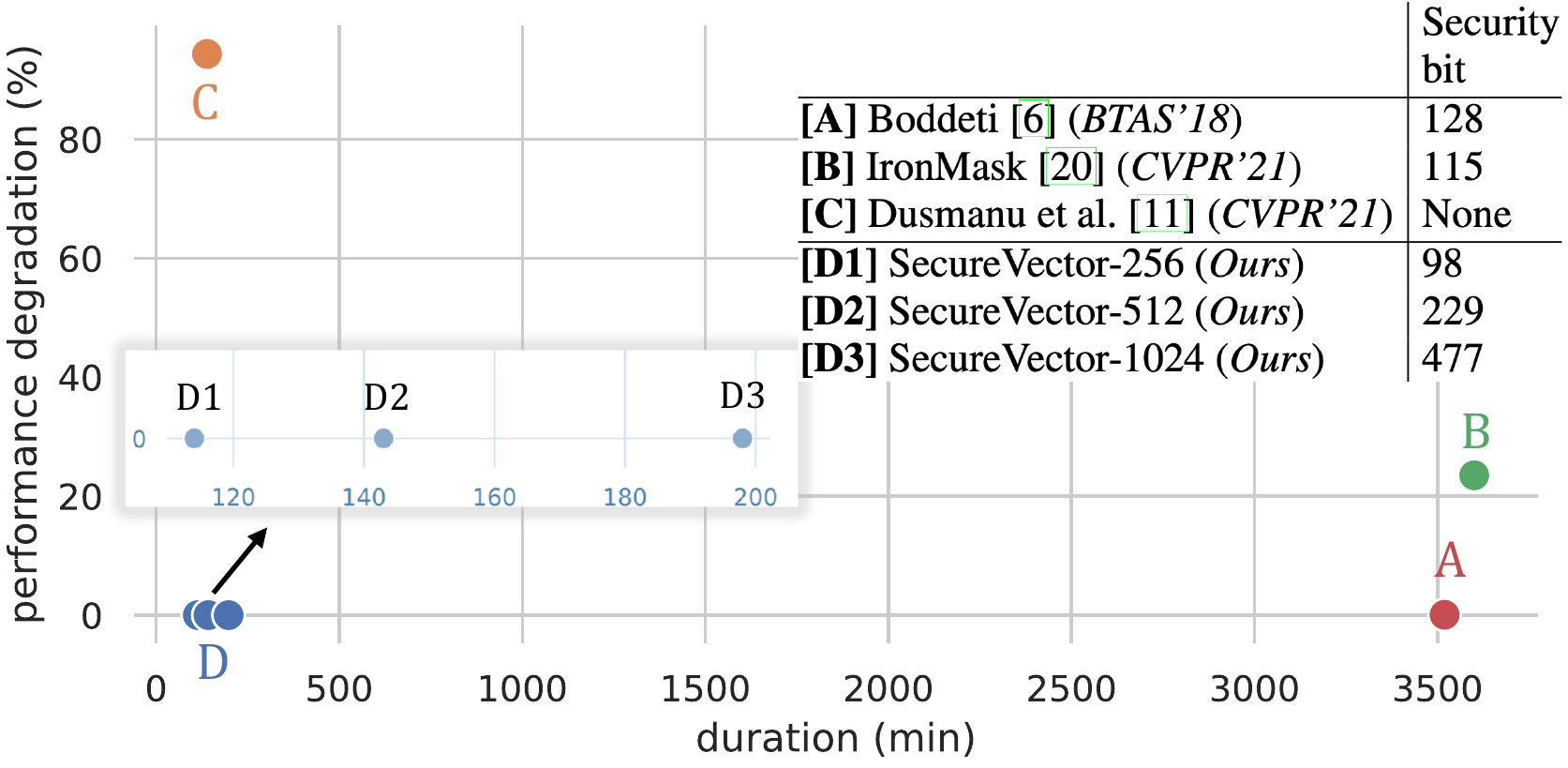}
  }
  \caption{
    \textbf{(a)}
    Given any vulnerable features $\bm x$, $\bm y$ extracted from sensitive data (\eg, facial images), SecureVector enrolls them by permuting into $\bm c_x$, $\bm c_y$ through random transformations $T_x, T_y$, which can be protected by a homomorphic encryption (HE) cryptosystem.
    Consequently, the nearest neighbor attack can no longer output faces from correct identities.
    For feature matching, the similarity between $\bm x, \bm y$ is calculated by $\bm c_x, \bm c_y$ and $T_z$, where $T_z$ is decrypted from summation of $T_x, T_y$ in the encryption domain.
    \textbf{(b)}
    We use features extracted from a released face recognition model~\cite{meng2021magface} and report performance degradation in terms of TAR@FAR=1e-4 on a large-scale benchmark called IJB-B~\cite{whitelam2017iarpa} on a machine with Intel(R) Xeon(R) CPU E5-2630 running at 2.20GHz and 256G RAM.
    Unlike the current methods suffering from significant performance degradations, heavy computational overheads, or no theoretical privacy guarantee, the proposed SecureVector can successfully lead to real-time and lossless feature matching along with much higher security bits.          
  }    
\end{figure}  

Visual retrieval techniques have been extensively deployed in real-world applications such as biometric (\eg, face, fingerprint and iris) recognition~\cite{meng2021magface, meng2021poseface, morampudi2020privacy, penn2014customisation, xu2021searching}, person/car re-identification~\cite{wang2017deep, hou2019interaction, gu2020appearance, meng2022basket} and image/video retrieval~\cite{radenovic2018fine, wray2021semantic, min2021convolutional}.
In such applications, visual sources are quantized as discriminative multi-dimensional feature vectors (\emph{a.k.a.} templates~\cite{patel2015cancelable, barni2015privacy, juels1999fuzzy, mohan2019significant, kim2021ironmask}) through two stages: the \emph{enrollment stage} to preprocess features along with corresponding labels for building a gallery, and the \emph{matching stage} to infer the classes of probes by calculating similarities of their features with those in the gallery.
Nonetheless, without being properly processed, these feature vectors are vulnerable to malicious attacks, posing severe threats to privacy.
For example, recent works~\cite{mai2018reconstruction, dong2021towards, meng2021learning} demonstrated that accessing face features enables a malicious attacker to reconstruct users' facial images, which are of sufficient quality to be identified by conventional face recognition systems.
Apart from the inversion attack, nearest neighbor attack may also inspect human identities as depicted in left panel of \cref{fig:intro}, where faces from the corresponding person can be deduced from the raw features $\bm x, \bm y$.
Similarly, sensitive facial attributes (\eg, ethnicity, age and gender) can be predicted from features with high accuracies~\cite{liu2015deep, dusmanu2021privacy, meng2021improving}.
The threats are not restricted to identification and recognition tasks, according to \citet{dusmanu2021privacy}.
In their trials, both deep features and hand-crafted local features such as SIFT~\cite{lowe2004distinctive} compromise privacy in the visual localization task.

The serious privacy issue has sparked a burgeoning academic field known as template protection~\cite{kim2021ironmask, dang2020fehash, mohan2019significant, patel2015cancelable, troncoso2013fully, boddeti2018secure, morampudi2020privacy, engelsma2020hers}, which sanitizes original feature vectors (templates) for the purpose of \textit{irreversibility}, \textit{renewability}, \textit{unlinkability} and \textit{performance retainability} as defined by ISE/IEC IS 24745~\cite{ISO24745}.
Irreversibility indicates that recovering original features from the sanitized data must be computationally impossible.
Renewability and unlinkability together contribute to the cancellable property. 
In particular, if one record gets leaked, it should be feasible to replace it with a renewed one that the adversary cannot link to the compromised record.
Performance retainability requires high verification accuracies during matching and thus guarantees high utility when protecting privacy, which is a critical concern by sensitive applications (\eg, biometric authentication).

Template protection has mostly evolved in two areas in recent years: transformation-based methods and cryptographic algorithms.
Early transformation approaches randomly assign variational codes for enrolled features and then fit mappings from features to corresponding codes~\cite{pandey2016deep, kumar2018face, dang2020fehash}.
However, the training-and-enrollment scheme scales poorly as re-fitting is ineluctable when adding or revoking features.
Several plug-in methods~\cite{mohan2019significant, kim2021ironmask, dusmanu2021privacy} make inspiring improvements by isolating each enrolled feature, but performance degradation cannot be circumvented due to the inevitable information loss during transformations.
Cryptographic algorithms are inherently plug-in approaches and usually utilize homomorphic encryption (HE)~\cite{gentry2012fully}, which enables arithmetic operations directly on ciphertexts.
Thus, feature distances can be computed in the encryption domain without compromising privacy.
Some successful attempts~\cite{upmanyu2010blind, barni2010privacy, penn2014customisation} have been reported for protecting binary feature vectors with partially homomorphic encryption (PHE).
Because PHE only supports one specific type of operation (\eg, addition for Paillier~\cite{paillier1999public} or multiplication for RSA~\cite{rivest1978method}), recent works~\cite{troncoso2013fully, boddeti2018secure, morampudi2020privacy, engelsma2020hers} resort to fully homomorphic encryption (FHE) for real-valued features, where addition and multiplication with unbounded depths are required.
FHE, on the other hand, consumes tremendous resources and therefore is impractical for large-scale applications.

Despite the progress, current plug-in methods still suffer from poor performances, heavy computational overheads, or limited security level as presented in \cref{fig:intro0}.
These drawbacks deprive their usages of real-world applications within utility-required, large-scale and sensitive scenarios.
To resolve these problems, this paper proposes a plug-in module called SecureVector for privacy-preserving, real-time, and lossless feature matching.
Unlike previous methods based on fully homomorphic encryption cryptosystems, we leverage on Paillier~\cite{paillier1999public}, a mature and efficient PHE allowing addition in the encrypted domain with the much higher security level and fewer resources consumption (\eg, computation, memory, and storage).
Given the real-valued features, the core idea of SecureVector is an enrollment scheme (the middle panel of \cref{fig:intro}) of random permutations controlled by sampled integers that can be protected by Paillier.
This process not only eliminates the need to process the original float values but also offers a large volume to reach high security bits, far exceeding existing state-of-the-arts as depicted in \cref{fig:intro0}.
During the matching phase, our design allows the addition of two permutations, which subsequently enables matching features $\bm x, \bm y$ by their permuted features $\bm c_x, \bm c_y$ and the addition of permutations $T_z$ in the right panel of \cref{fig:intro}.
That property waives the necessity for decrypting any individual record, preventing the recovery of original features.
The permutation scheme enables privacy-preserving and lossless feature matching.
However, there remain too many integers in each permutation to be processed by Paillier.
Therefore, we further propose a 4L-DEC converter that compresses all integers into a single value and vice versa while keeping the additional utility in the encryption domain.
By this means, Paillier only performs one time of the encryption for enrollment and one time of addition plus decryption for matching features.
As a result, enrolling a 512-dim vector in SecureVector only takes 0.59 ms and requires 0.59 Kb of memory, while the numbers for feature matching are 0.30 ms and 10.81 Kb.
Its efficiency is even better than \citet{dusmanu2021privacy}, a recent work that relies solely on linear algebra but suffers from poor matching results and no theoretical privacy guarantee.
Our contributions can be summarized as
\textbf{(a)}
To our best knowledge, SecureVector is the first to achieve privacy-preserving, real-time and lossless feature matching, whereas current state-of-the-arts can only fulfill parts of these properties.
The scheme of permutation, a 4L-DEC converter together with the existing Paillier cryptosystem are used for secure feature matching.
Rigorous theoretical proofs of security properties are provided as well.
\textbf{(b)}
We conduct extensive experiments on multiple tasks, including face recognition, person re-identification, image retrieval, and privacy analyses.
The security bit of the used SecureVector-512 is almost twice that of the best value of the baselines.
Aside from lossless feature matching, SecureVector outperforms the competitors in term of running time, memory consumption, and storage usage.
These results have demonstrated the effectiveness of our method.


\section{Related Works}

\noindent\textbf{Biometric Template Protection:}
To sanitize sensitive information, several biometric template protection efforts transform original features by non-invertible projections.
Cancellable biometrics~\cite{patel2015cancelable}, key-binding~\cite{barni2015privacy, juels1999fuzzy}, and affine subspaces~\cite{dusmanu2021privacy} are a few of the successful approaches that have showed promise in the field.
Recently, the main body of researches~\cite{sutcu2007protecting, kumar2018face, pandey2016deep, mai2020secureface, dang2020fehash, mohan2019significant, kim2021ironmask} projects biometric templates into random error-correcting-codes, which are further secured by hash functions.
For example, \citet{pandey2016deep} and \citet{kumar2018face} leverage deep neural networks to fit mappings between templates and maximum entropy binary codes.
To further accord with the embedding properties, \citet{mohan2019significant} encodes features by significant bits while FEHash~\cite{dang2020fehash} generates discriminative codes by iteratively separating features via hashing hyperplanes.
As binarization is the leading cause of performance drop, IronMask~\cite{kim2021ironmask} designs real-valued codes compatible with angular distance metric.
Despite these breakthroughs, performance degradation is still inevitable due to information loss during mapping.
This prevents these methods from being adopted in real-world scenarios, especially for sensitive fields like biometric authentication.

\noindent\textbf{Homomorphic Encryption for Biometrics:}
Homomorphic encryption (HE)~\cite{gentry2012fully} enables basic arithmetic operations directly on ciphertext and, as such, can be utilized to compute distances between two encrypted feature vectors.
Based on supported arithmetic operations, HE can be categorized into partially homomorphic, somewhat homomorphic, leveled fully homomorphic, and fully homomorphic encryption~\cite{armknecht2015guide}.
Partially homomorphic encryption (PHE), which only supports one type of operation (\eg, addition for Paillier~\cite{paillier1999public} and multiplication for  RSA~\cite{rivest1978method}), has several successful attempts for template protection.
For example, \citet{upmanyu2010blind} computes scores by RSA with repeated rounds of communications between the client and the database.
\citet{barni2010privacy} and \citet{penn2014customisation} utilize Paillier on fields of fingerprint and iris recognition, respectively.
However, these approaches are specific for binary templates and cannot be extended to real-valued features, where addition and multiplication are both required to compute distances between features.
A direct solution to the issue is the fully homomorphic encryption (FHE), the strongest notion of HE, which enables both operations with unbounded depths.
Although recent works~\cite{troncoso2013fully, boddeti2018secure, morampudi2020privacy, engelsma2020hers} reveal potentials of FHE, the heavy computational overheads still limit their usages in real-world circumstances.

\section{Methodology}
We first revisit Paillier in \cref{sec:method_paillier}.
\cref{sec:method_sv} details enrollment and matching stages, following which we present hyper-parameter studies and recommended settings in \cref{sec:select_param}.
In the end, we demonstrate how the SecureVector satisfies the ISO/IEC IS 24745~\cite{ISO24745} in \cref{sec:method_an}.

\subsection{Paillier Cryptosystem Revisited}\label{sec:method_paillier}

Paillier~\cite{paillier1999public} is an additive homomorphic cryptosystem based on the decisional composite residuosity assumption, which is believed to be intractable.
The main parameter of Paillier is the key size $S$ representing the security bit of the algorithm.
A big $S$ denotes a strong security level where $2^S$ possible attempts are required for an adversary to crack the system.
Paillier easily achieves a higher level (\eg, above 512) with far fewer resources than fully homomorphic encryption schemes (\eg, BFV~\cite{fan2012somewhat} and CKKS~\cite{cheon2017homomorphic}), which are normally limited to at most 128 or 192 security bits
\cref{sec:addexp_he} offers detailed comparisons of these schemes.

Given a key size $S$, Paillier first generates a pair of public ($pub$) and private ($pri$) keys.
Suppose  $m_1$, $m_2$ are two integers and their ciphertexts are $\tilde{m}_1$, $\tilde{m}_2$.
Paillier provides \textit{encryption} and \textit{decryption} functions with \textit{homomorphic addition} promised.
By denoting the encryption ($E$) and decryption ($D$) operations as $E_{pub}(m_1) = \tilde{m}_1$ and $D_{pri}(\tilde{m}_1) = m_1$ respectively, the property of homomorphic addition enables $D_{pri}(\tilde{m}_1 * \tilde{m}_2) = m_1 + m_2$, where $*$ is an operation in encryption domain.
Unlike \citet{xiao2012efficient} or other symmetric systems, Paillier is asymmetric as only public key is required for encryption, which further enhances its security.
Please refer to the original paper~\cite{paillier1999public} for details.

\subsection{SecureVector}\label{sec:method_sv}

Despite its theoretical superiority, Paillier cannot be directly employed to protect real-valued high-dimensional features because of three reasons:
(1) Only integer values are supported in the original algorithm.
(2) Matching features involves both addition and multiplication operations, while Paillier only supports the former one.
(3) Its linear growth rate on resource consumption concerning the size of encrypted/decrypted values and the number of operations in the encryption domain would be impractical to handle high-dimensional features (\eg, 512 is the typical length for a face recognition feature).
To bridge the gap between Paillier and the field of feature matching, we propose SecureVector with the enrollment and matching stages as designed in \cref{fig:method}.
We follow the common setting where features are $l_2$ normalized with cosine similarity as the matching metric.
With a few modifications presented in \cref{sec:ext_match}, our method can be easily extended for general features or other metrics.

\subsubsection{Enrollment in SecureVector}\label{sec:enrollment}

In theory, Paillier lacks the ability to directly handle real-valued feature vectors that would be further matched through addition and multiplication operations.
Instead, we resort to a random permutation scheme and a 4L-DEC encryptor as illustrated in \cref{fig:method}(a).
Raw features are sanitized via random signals composed of integers.
These signals are additive and thereby Paillier-friendly.
Then the 4L-DEC encryptor compresses signals into one integer with the additive property kept and encrypts outputs via Paillier.

\begin{figure*}[htb!]
  \centering
  \includegraphics[width=1.1\textwidth]{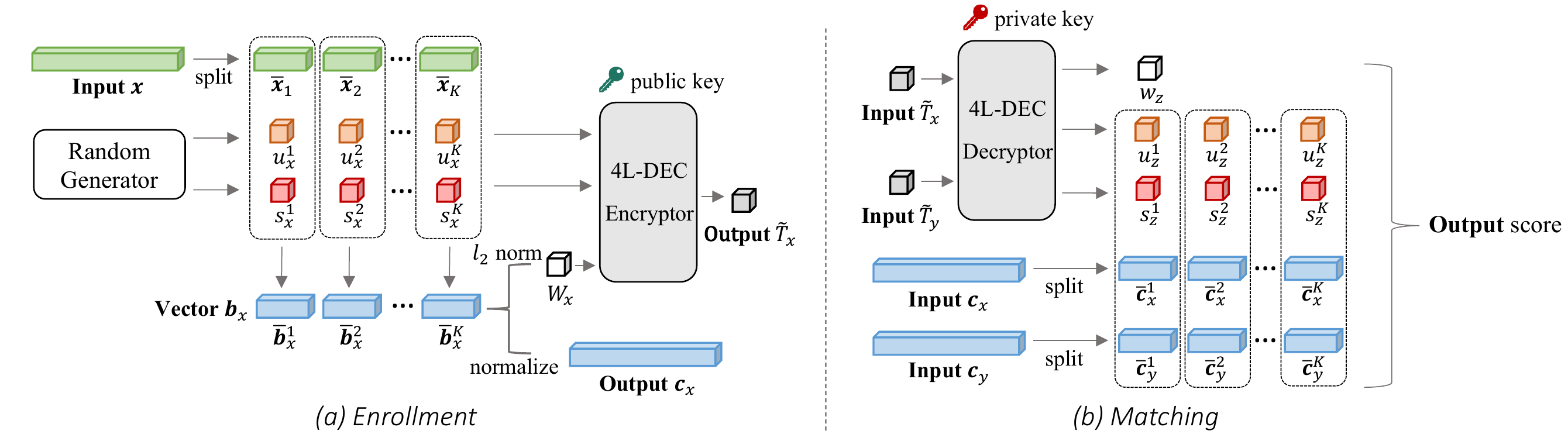}
  \caption{Workflow of SecureVector.
    (a) During enrollments, the input $\bm x$ is split into $\{\bar{\bm x}_i\}_{i=1}^{K}$ and permuted by random generated $\{u_x^i\}_{i=1}^K$, $\{s_x^i\}_{i=1}^K$ for value scaling and sign flipping.
    This permutation outputs $\bm b_x$, which is then normalized by its $l_2$ norm and produces a sanitized vector $\bm c_x$.
    After quantizing the $l_2$ norm as $W_x$, a 4L-DEC encryptor encrypts $W_x$ and $\{u_x^i,s_x^i\}_{i=1}^K$  into $\tilde{T}_x$, which together with $\bm c_x$ form the enrolled outputs.
    (b) Suppose $\{\bm c_x, \tilde{T}_x\}$ and $\{\bm c_y, \tilde{T}_y\}$ are records to match.
    With $\tilde{T}_x$  and $\tilde{T}_y$ as inputs,  the 4L-DEC decryptor generates $w_z$ and $\{s_z^i,u_z^i\}_{i=1}^K$ without compromising individual information.
    Then the score can be calculated by $\bm c_x, \bm c_y$, and the decrypted results.
  }\label{fig:method}
\end{figure*}

\noindent\textbf{Random permutation.}
Given a feature $\bm x = (x_1 \cdots, x_d)$, we permute it with the following steps.
First, we divide $\bm x$ into $K$ segments with equal length $\bar d$ (\ie, $K\bar d = d$) as $\bm x = (\bar {\bm x}_1, \cdots, \bar {\bm x}_K)$.
We then permute each $\bar {\bm x}_i$ into $\bar {\bm b}_x^i =  s_x^i \cdot e^{(u_x^i-L)/M}\cdot \bar {\bm x}_i$ where the sampled $u_x^i \in \{0,1, \cdots, 2L-1\}$ randomly scales values and $s_x^i\in \{-1, 1\}$ flips the sign of $\bar{\bm x}_i$.
The hyper-parameters $L, M$ are positive integers, controlling the scaling resolution.
As elaborated later, this design enables the matching between enrolled features.
In a nutshell, we permute $\bm x$ into $\bm c_x$ with:
\begin{equation}
  \small
  \label{eq:b2c}
  \begin{aligned}
    \bm c_x = \frac{ \bm b_x}{W_x}, \text{ where }  W_x = \|\bm b_x\|_2, \bm b_x &= (\bar {\bm b}_x^1, \bar {\bm b}_x^2, \cdots,\bar {\bm b}_x^K)
  \end{aligned}
\end{equation}

\noindent\textbf{4L-DEC Encryptor.}
Our next step is to encrypt $W_x$ and $\{u_x^i, s_x^i\}_{i=1}^K$ to prevent the recovery of original feature $\bm x$ from $\bm c_x$.
Directly using Paillier is computationally expensive as the cost grows linearly with the number of values to encrypt.
To address this issue, we design a 4L-DEC encryptor that compresses all values into a single 4L-bit integer while avoiding interferences during addition in the encryption domain.
The encryption consists of the following steps:

\quad\textbf{\textit{(1)}}
To encrypt all $\{s_x^i\}_{i=1}^K$ in a unified code, we randomly sample $v_x^i$ as an even number from $\{0, \cdots, 2L-1\}$ if $s_x^i=1$ or an odd one if $s_x^i=-1$.

\quad\textbf{\textit{(2)}}
We must quantize $W_x$ in advance as it's a float value.
As we have $\log W_x \in [-L/M, L/M)$ according to \cref{lemma:wf} in \cref{sec:ext_bound}, we equally divide the range into $2^{15}\cdot L^8$ parts, which is sufficiently large for high quantization precision.
In the end, we let $w_x = \lfloor \frac{\log W_x + L/M}{2L/M}\cdot 2^{15}\cdot L^8 \rfloor$.

\quad\textbf{\textit{(3)}}
We compress $\{u_x^i, v_x^i\}_{i=1}^K$ and $w_x$ into a single decimal number by concatenating them sequentially in the 4L-bit number system:
\begin{equation}
  \small
  \label{eq:uvw2c}
  T_x = \sum_{i=1}^{K} u_x^i \cdot (4L)^{i-1} + \sum_{i=1}^{K} v_x^i \cdot (4L)^{i+K-1} + w_x \cdot (4L)^{2K}.
\end{equation}

\quad\textbf{\textit{(4)}}
In the end, we encrypt the number as $\tilde T_x = E_{pub}(T_x)$ via Paillier.
It's worth noting that we have compressed $2K+1$ numbers into one integer.
This design highly speeds up the whole algorithm as the operations of encryption, addition, and decryption in Paillier are all reduced $2K+1$ times.

\subsubsection{Properties of the Designed Permutation}

This part analyzes the designed permutation scheme.
Assume that there are two features $\bm x, \bm y$, and that the enrolled results are $\{\bm c_x, \tilde T_x\}, \{\bm c_y, \tilde T_y\}$ respectively.
Suppose $T_x$ are generated by $\{u_x^i, v_x^i\}_{i=1}^K$ and $w_x$, and $T_y$ are generated by $\{u_y^i, v_y^i\}_{i=1}^K$ and $w_y$.
With Paillier of key size $S$, we can compute the sum of $\{T_x, T_y\}$ by $T_z = D_{pri}(\tilde T_x * \tilde T_y)$.
We start with two properties of the adding process.

\begin{lemma}
  \label{lemma:carry_prop}
  There is no carry-propagation for the first 2K digits during the addition of $T_x, T_y$ in the 4L-bit number system.
\end{lemma}
\begin{proof}
  As $\{u_x^i, v_x^i, u_y^i, v_y^i\}_{i=1}^K$ are sampled from $\{0, 1, \cdots, 2L-1\}$,  then $0 \leq (u_x^i+u_y^i), (v_x^i + v_y^i) \leq 4L-2 < 4L-1$ is always true.
  We can further have
    $T_x+T_y = \sum_{i=1}^K (u_x^i + u_y^i)\cdot (4L)^{i-1} +  \sum_{i=1}^K (v_x^i+v_y^i) \cdot (4L)^{K+i-1} + (w_x+w_y) \cdot (4L)^{2K}$ according to \cref{eq:uvw2c}.
  As a result, no carry-propagation occurs for the first 2K digits in 4L-bit number system, which corresponds to  $\sum_{i=1}^K (u_x^i + u_y^i)\cdot (4L)^{i-1} +  \sum_{i=1}^K (v_x^i+v_y^i) \cdot (4L)^{K+i-1}$.
\end{proof}

\begin{lemma}
  \label{lemma:sum_bound}
  In the 4L-bit number system, the addition of $T_x, T_y$ can be expressed with $2K+8$ digits.
  In other words, $T_x + T_y < (4L)^{2K+9}$ always holds.
\end{lemma}
\begin{proof}
  \cref{lemma:carry_prop} presents $\leq (u_x^i+u_y^i), (v_x^i + v_y^i) < 4L-1$. 
  Similarly, $w_x, w_y \in \{0, 1, \cdots,  2^{15}L^8-1\}$ further yeilds $w_x + w_y \leq 2^{16}L^8 -2 < (4L)^8-1 = (4L-1)\cdot\left(\sum_{i=1}^8(4L)^{i-1}\right)$. 
  Then we have
  \begin{equation*}
    \small
    \label{eq:sum_bound}
    \begin{split}
      T_x+T_y & = \sum_{i=1}^K (u_x^i + u_y^i)\cdot (4L)^{i-1} +  \sum_{i=1}^K (v_x^i+v_y^i) \cdot (4L)^{K+i-1} + (w_x+w_y) \cdot (4L)^{2K}    \\
      & < (4L-1)\cdot\sum_{i=1}^{2K} (4L)^{i-1} + (4L-1)\cdot\left(\sum_{i=1}^8(4L)^{i-1}\right) \cdot (4L)^{2K}  =  \sum_{i=1}^{2K+8}(4L-1)\cdot (4L)^{i-1} 
    \end{split}
  \end{equation*}
  The upper bound is $\sum_{i=1}^{2K+8}(4L-1)\cdot (4L)^{i-1} = (4L)^{2K+9} - 1$, which completes the proof.
\end{proof}

Because Paillier's capacity is $2^S$, we must prevent numerical overflow in order to successfully encrypt $T_x, T_y$, add them in the encryption domain and decrypt the sum.
\cref{lemma:keysize} circumvents the risk.
\begin{lemma}
  \label{lemma:keysize}
  Paillier with key size $S\geq(2K+9)\log_2(4L)$ is sufficient to avoid numerical overflow.
\end{lemma}  
\begin{proof}
  If $S\geq(2K+9)\log_2(4L)$, we get $2^s \geq 4L^{2K+9} > 4L^{2K+9}-1> T_x+T_y \geq T_x, T_y$, which means the capacity of Paillier is greater than $T_x, T_y$   as well as their sum.
\end{proof}

Previous analyses guarantee the accurate output of $T_z=D_{pri}(\tilde T_x *\tilde T_y) = T_x+T_y$.
Then we extract the corresponding bits $\{u_z^i, v_z^i\}_{i=1}^K, w_z$ from $T_z$ by
\begin{equation}
  \small
  \label{eq:get_uvw_z}
  u_z^i  = \lfloor T_z/(4L)^{i-1}\rfloor \mod 4L, \quad v_z^i  = \lfloor T_z/(4L)^{K+i-1}\rfloor \mod 4L, \quad w_z  = \lfloor T_z/(4L)^{2K}\rfloor.
\end{equation}
We next show that the designed permutation scheme is additive.

\begin{lemma}
  \label{lemma:uvw}
  For $\{u_z^i, v_z^i\}_{i=1}^K, w_z$, the following equations hold:
  \begin{equation}
    \small
    \label{eq:proof_uvw_z}
      u_z^i = u_x^i + u_y^i, \quad v_z^i = v_x^i + v_y^i,\quad  w_z  = w_x + w_y,\qquad \forall i \in \{1, 2,\cdots K\}
  \end{equation}
\end{lemma}
\begin{proof}
  Let's start by proving that $w_z = w_x + w_y$.
  We have
  $w_z = \lfloor T_z/(4L)^{2K}\rfloor = \lfloor \left(\sum_{j=1}^K (u_x^j + u_y^j)\cdot (4L)^{j-1}+ \sum_{j=1}^K (v_x^j+v_y^j) \cdot (4L)^{K+j-1} + (w_x+w_y) \cdot (4L)^{2K}\right)/(4L)^{2K} \rfloor  = \lfloor \left(\sum_{j=1}^{K} (u_x^j + u_y^j)\cdot (4L)^{j-i} + \sum_{j=1}^K (v_x^j+v_y^j) \cdot (4L)^{K+j-1}\right)/(4L)^{2K}  \rfloor + (w_x+w_y)$ as $T_z = T_x+T_y$.
  It's simple to deduce $\sum_{j=1}^{K} (u_x^j + u_y^j)\cdot (4L)^{j-i} + \sum_{j=1}^K (v_x^j+v_y^j) \cdot (4L)^{K+j-1} < (4L)^{2K}$ from \cref{lemma:carry_prop}.
  After rounding down, the ratio in the first term equals 0, resulting in $w_z = w_x+w_y$.
  
  Steps below can prove that $u_z^i = u_x^i + u_y^i$ (similarly we can get $v_z^i = v_x^i + v_y^i$):
  \begin{equation*}
    \small
    \begin{aligned}
      u_z^i =& \lfloor T_z/(4L)^{i-1}\rfloor \mod 4L \\
      =&
      \lfloor \left(\sum_{j=1}^K (u_x^j + u_y^j)\cdot (4L)^{j-1}+ \sum_{j=1}^K (v_x^j+v_y^j) \cdot (4L)^{K+j-1} \right.\left. + (w_x+w_y) \cdot (4L)^{2K}\right)/(4L)^{i-1} \rfloor
      \mod 4L \\
      =& \left(\lfloor (\sum_{j=0}^{i-1} (u_x^j + u_y^j)\cdot (4L)^{j})/(4L)^i \rfloor + 
        \sum_{j=i}^K (u_x^j + u_y^j)\cdot (4L)^{j-i}+ \sum_{j=1}^K (v_x^j+v_y^j) \cdot (4L)^{K+j-i} \right.\\
      & \left.+ (w_x+w_y) \cdot (4L)^{2K-i+1}\right) \mod 4L = (0+u_x^i + u_y^i + 4L \cdot \psi_u)\mod 4L 
    \end{aligned}
  \end{equation*}
  Here $\psi_v=\sum_{j=i+1}^K (v_x^j + v_y^j)\cdot (4L)^{j-i-1}+(w_x+w_y) \cdot (4L)^{K-i}$ is an integer, which makes $4L \cdot \psi_u\mod 4L = 0$.
  Moreover, we have $(u_x^i+u_y^i) \mod 4L = u_x^i+u_y^i$ because the sum will not exceed $4L-2$.
  In the end, we finish the proof by $u_z^i = (0+u_x^i + u_y^i + 4L \cdot \psi_u)\mod 4L  = u_x^i+u_y^i$.
\end{proof}

\begin{lemma}
  \label{lemma:s}
  For $i = 1,2, \cdots, K$, $s_z^i = s_x^i\cdot s_y^i$ always holds if $s_z^i$ is defined as 
  \begin{equation}
  \small
  \label{eq:2}
  s_z^i  = \left \{
    \begin{aligned}
      1  & \text{\quad    if $v_z^i $ is even}, \\
      -1 & \text{\quad    if $v_z^i $ is odd}. \\
    \end{aligned}
  \right.
\end{equation}
\end{lemma}
\begin{proof}
  \cref{lemma:uvw} presents $v_z^i = v_x^i + v_y^i$.
  If $v_z^i$ is even, then $v_x^i, v_y^i$ are both even or odd, implying that $s_x^i, s_y^i$ are both 1 or -1.
  Then we have $s_x^i \cdot s_y^i = 1$. 
  If $v_z^i$ is odd, then $v_x^i, v_y^i$ contrains one even and one odd integer.
  As a result, $s_x^i \cdot s_y^i = -1\cdot 1 = -1$ is true. 
  $s_z^i = s_x^i \cdot s_y^i$ is correct in both cases.
\end{proof}

\subsubsection{Feature Matching in SecureVector}
\cref{fig:method}(b) describes the feature matching process.
With two features $\bm x, \bm y$ and their enrolled results are $\{\bm c_x, \tilde T_x\}, \{\bm c_y, \tilde T_y\}$, 4L-DEC decryptor takes $\tilde T_x, \tilde T_y$ as the input and produces $\{u_z^i, s_z^i\}_{i=1}^K$ and $w_z$.
Then the score can be calculated by $\bm c_x, \bm c_y$ and the outputs of the 4L-DEC decryptor.

\noindent\textbf{4L-DEC Decryptor.}
If $\tilde T_x, \tilde T_y$ are decrypted directly, an adversary can potentially recover the original features if $T_x, T_y$ are exposed.
Instead, the proposed 4L-DEC decryptor only decrypts the summations (\ie, $T_z = D_{pri}(\tilde T_x *\tilde T_y)$), and then decompresses $T_z$ into $\{u_z^i, v_z^i\}_{i=1}^K$ and $w_z$.
These steps bypass the aforementioned privacy risk as proved in \cref{lemma:attack} in appendix and are detailed as follows:

\quad\textbf{\textit{(1)}}
With the private key, we decrypt the summation $T_z = D_{pri}(\tilde T_x *\tilde T_y)$.

\quad\textbf{\textit{(2)}}
$\{u_z^i, v_z^i\}_{i=1}^K$ and  $w_z$ are generated from $T_z$ by \cref{eq:get_uvw_z}.

\quad\textbf{\textit{(3)}}
For $i=1,2,\cdots, K$,  we set $s_z^i$ to be 1 if $v_z^i $ is even and -1 otherwise, as stated in \cref{eq:2}. 

\noindent\textbf{Matching score.}
In the end, the score is calculated by the following lemma:

\begin{lemma}
  \label{lemma:score}
  The similarity score of features $\bm x, \bm y$ can be calculated by $\{\bm c_x, \bm c_y\}$ and generated $\{u_z^i, s_z^i\}_{i=1}^K, w_z$ with
  $\bm x^T\bm y=\sum_{i=1}^K  \frac{e^{\frac{w_z - 2^{15}L^8}{2^{14}L^7M}}}{s_z^i\cdot e^{(u_z^i-2L)/M}} \cdot (\bar {\bm c}_x^i)^T\bar {\bm c}_y^i$.
\end{lemma}

\begin{proof}
  Because we have $w_x = \lfloor \frac{\log W_x + L/M}{2L/M}\cdot 2^{15}\cdot L^8 \rfloor$, it's easy to recover the $W_x$ from $w_x$ by $\log W_x \approx \frac{w_x}{2^{14}L^7M} - \frac{L}{M} = \frac{w_x - 2^{14}L^8}{2^{14}L^7M}$.
  Then the cosine similarity $\bm x^T\bm y = \sum_{i=1}^K \bar {\bm x}_i^T\bar {\bm y}_i$ is
  \begin{equation*}
    \small
    \label{eq:recover}
    \begin{aligned}
      \bm x^T \bm y & = \sum_{i=1}^K \frac{(\bar {\bm b}_x^i)^T}{s_x^i \cdot e^{(u_x^i-L)/M}}\frac{\bar {\bm b}_y^i}{s_y^i \cdot e^{(u_y^i-L)/M}}  = \sum_{i=1}^K \frac{W_x \cdot (\bar {\bm c}_x^i)^T}{s_x^i \cdot e^{(u_x^i-L)/M}}\frac{W_y\cdot \bar {\bm c}_y^i}{s_y^i \cdot e^{(u_y^i-L)/M}}
      & \text{\ttfamily by \cref{eq:b2c}} \\
      & = \sum_{i=1}^K  \frac{W_x\cdot W_y}{s_x^is_y^i\cdot e^{(u_x^i+u_y^i-2L)/M}}\cdot (\bar {\bm c}_x^i)^T\bar {\bm c}_y^i  = \sum_{i=1}^K  \frac{e^{(\log W_x + \log W_y)}}{s_x^is_y^i\cdot e^{(u_z^i-2L)/M}} \cdot (\bar {\bm c}_x^i)^T\bar {\bm c}_y^i & \text{\ttfamily by \cref{lemma:uvw}} \\
      & = \sum_{i=1}^K  \frac{e^{(\log W_x + \log W_y)}}{s_z^i\cdot e^{(u_z^i-2L)/M}} \cdot (\bar {\bm c}_x^i)^T\bar {\bm c}_y^i  \approx \sum_{i=1}^K  \frac{e^{\frac{w_x - 2^{14}L^8}{2^{14}L^7M} + \frac{w_y - 2^{14}L^8}{2^{14}L^7M}}}{s_z^i\cdot e^{(u_z^i-2L)/M}} \cdot (\bar {\bm c}_x^i)^T\bar {\bm c}_y^i
      & \text{\ttfamily by \cref{lemma:s}} \\
      & = \sum_{i=1}^K  \frac{e^{\frac{w_x+w_y - 2^{15}L^8}{2^{14}L^7M}}}{s_z^i\cdot e^{(u_z^i-2L)/M}} \cdot (\bar {\bm c}_x^i)^T\bar {\bm c}_y^i = \sum_{i=1}^K  \frac{e^{\frac{w_z - 2^{15}L^8}{2^{14}L^7M}}}{s_z^i\cdot e^{(u_z^i-2L)/M}} \cdot (\bar {\bm c}_x^i)^T\bar {\bm c}_y^i & \text{\ttfamily by \cref{lemma:uvw}}
    \end{aligned}
  \end{equation*}
\end{proof}

\subsection{Selections of Hyper-parameters}
\label{sec:select_param}

Central to the performance of SecureVector is the choice of hyper-parameters $\{K, L, M\}$.
This part introduces three principles for a better parameter selection in practice.

\noindent\textbf{Permutation Degree.} 
A dis-similar feature $\bm c_x$ permuted from the original $\bm x$ is preferred to conceal privacy.
The downside of a heavy permutation is the substantial numerical error for the quantized $w_x$.
Based on experimental studies in \cref{sec:apexp_param}, we recommend $M = L/128$ for a better trade-off.

\noindent\textbf{Numerical Capcacity.}
Key size $S\geq(2K+9)\log_2(4L)$ can satisfy the capacity according to \cref{lemma:keysize} .
Therefore, $L$ can be constrained as $L \leq  2^{S/(2K+9)-2}$ to satisfy the requirement.

\noindent\textbf{Security Level.} 
The security level should be as large as possible for the sake of privacy.
In our method, a feature $\bm x$ is permuted into $\bm c_x$ by $\{u_x^i, s_x^i\}_{i=1}^K$.
Recovering $\bm x$ with a brute-force attack requires traversing all $2^K\cdot(2L)^K$ possible combinations, which leads to $2K + K\cdot \log_2(L)$ security bits.
Under a specific key size $S$ of Paillier, we aim to find the optimal $K$ for the best security bit  ($b$).
Following the second principle, taking the upper bound $L$ yeilds the highest corresponding security bit as $b = 2K + K\cdot \log_2(\lfloor 2^{S/(2K+9)-2} \rfloor)$.
Because a feature is segmented into $K$ equal parts, a feasible $K$ should be a divisor of the feature dimension $d$, which is normally set as powers of two (\eg, 512).
In the end, we present the recommended settings for $K$ under different key sizes in \cref{tab:invisible_securitybit}.
{
  \setlength{\tabcolsep}{3.6pt}
  \begin{table}[htb!]
    \centering
    \footnotesizea
    \caption{The optimal settings under different key sizes $S$.
    }    \label{tab:invisible_securitybit}
    \begin{tabular}{c|cccccc}
      \hline
      Key size $S$ & 128 & 256 & 512 & 1024 & 2048 & 4096 \\
      The best $K$ & 16 & 16 & 64 & 64 & 128 & 256 \\
      \hline
      Security Level (bits) & 48 & 98 & 229 & 477 & 989 & 2011 \\
      \hline
    \end{tabular}
  \end{table}
}

\subsection{Security Analyses}\label{sec:method_an}
A template protection scheme should meet \textit{Irreversibility}, \textit{renewability}, \textit{unlinkability} and \textit{performance retainability} as stated by ISO/IEC IS 24745~\cite{ISO24745}.
In this section, we present security analyses to demonstrate that SecureVector can successfully satisfy these requirements.

\noindent\textbf{Irreversibility}.
In SecureVector, a feature $\bm x$ is enrolled into $\{\bm c_x, \tilde T_x\}$, where the latter element $\tilde T_x$, protected by Paillier,  is confidently assumed to be secure.
The design permutation further equips SecureVector with a high security bit (\eg,  \cref{fig:intro0}), which contribute to the irreversibility property.

\noindent\textbf{Renewability}.
Once $\{\bm c_x, \tilde T_x\}$ is under the risk of compromise, we only need to re-sample a permutation and generate a new $\{\bm c_x', \tilde T_x'\}$ following the enrollment steps.

\noindent\textbf{Unlinkability}.
Various permutations enroll $\bm x$ into different $\{\bm c_x, \tilde T_x\}$.
Even the same permutation can lead to different $\tilde T_x$ because of the inherent mechanism in Paillier.
That reaches the unlinkability.

\noindent\textbf{Performance retainability}.
If ignoring numerical errors, only the quantization of $W_x$ introduces errors in SecureVector.
We relieve the quantization errors by dividing the range into a sufficient number of intervals (\ie, $2^{15}\cdot L^8>32,000$).
As a result, the performances are highly preserved.

\section{Experiments}
\label{sec:exp}

In the main text, we compare SecureVector with current plug-in template protection methods on face recognition.
In \cref{sec:exp_ext_face,sec:exp_reid,sec:exp_ir}, we show extra face-related results as well as experiments on person re-identification and image retrieval tasks.
SecureVector's resistance to malicious attacks is verified by \cref{sec:exp_pa}.
All these results reveal the superiority and generality of our method.

\noindent
\textbf{Baselines.}
We implement several recent plug-in methods including IronMask~\cite{kim2021ironmask}, \citet{boddeti2018secure} and \citet{dusmanu2021privacy} with recommended parameters used in the corresponding papers.
Specifically, we use $\alpha=16$ for IronMask~\cite{kim2021ironmask} and set the precision to be 125 for \citet{boddeti2018secure}.
Hybrid lifting mode with dimension 4 is used for \citet{dusmanu2021privacy}, which balances the accuracy and privacy as stated by the authors.
For our SecureVector, we use key size 512 and denote it as SecureVector-512.

\noindent
\textbf{Evaluation.}
We evaluate these methods by comparing security bits, reporting recognition performances on popular benchmarks (\ie, LFW~\cite{huang2008labeled}, CFP-FP~\cite{sengupta2016frontal}, AgeDB-30~\cite{moschoglou2017agedb}, IJB-B~\cite{whitelam2017iarpa} and IJB-C~\cite{maze2018iarpa}) and examining resource costs (\ie, running time, memory and storage consumptions).
For fair comparisons, methods are tested by the same machine with Intel(R) Xeon(R) CPU E5-2630 running at 2.20GHz and 256G RAM.
Features are extracted by a open-sourced ResNet100 model\footnote{https://github.com/IrvingMeng/MagFace.}.

{
  \setlength{\tabcolsep}{1pt}
  \begin{table*}[htb!]
    \centering
    \caption{Performances of plug-in feature protection methods (Prot.: Protection).
    }\label{tab:main}
    \footnotesizea
    \begin{tabular}{lllllllrrrrrrr}
      \hline
      \multirow{2}{*}{Method} & Security & Full &\multicolumn{3}{c}{Verification Accuracy (\%)}&& \multicolumn{2}{c}{Avg. Time(ms)} && \multicolumn{2}{c}{Avg. Memory(Kb)} & Storage  \\
      \cline{4-6} \cline{8-9} \cline{11-12} 
                              & Bit & Prot. &LFW& CFP-FP & AgeDB  && Enroll& Match && Enroll& Match & (Kb) \\
      \hline
      Direct &- & -&99.82 &98.41 &98.28 & & - &- & &- &- & -\\
      \hline
      IronMask~\cite{kim2021ironmask} &115  & no &89.30{ (-10.52)} &58.01{ (-40.04)} &57.87{ (-40.41)} &&557.92 &0.47 &&2040.71 &2061.33 &2100 \\
      \citet{boddeti2018secure} &128 & yes &99.82{ (\textbf{-0.00})} &98.34{ (-0.07)} &98.23{ (-0.05)} && 2.16 &23.6 &&24.44 &451.42 &87\\
      \citet{dusmanu2021privacy} & \textbf{NA} & yes&98.33{ (-1.49)} &88.51{ (-9.89)} &89.23{ (-9.05)} &&1.01 & \textbf{0.16} &&11.03 &42.96 &21 \\
      \hline
      SecureVector-512 & \textbf{229} & yes &99.82{ (\textbf{-0.00})} &98.41{ (\textbf{-0.00})} &98.28{ (\textbf{-0.00})} &&\textbf{0.59} & 0.30 && \textbf{0.59} & \textbf{10.81} & \textbf{5} \\
      \hline
    \end{tabular}
  \end{table*}
}

{
  \setlength{\tabcolsep}{2pt}
  \begin{table*}[htb!]
    \centering
    \caption{Verification accuracy (\%)  on challenging benchmarks.
    }\label{exp:tab_ijb}
    \footnotesizea
    \begin{tabular}{llllllll}
      \hline
      \multirow{2}{*}{Method} & \multicolumn{3}{c}{IJB-B (TAR@)} && \multicolumn{3}{c}{IJB-C (TAR@)} \\
      \cline{2-4} \cline{6-8}
                              &  FAR=1e-4 & FAR=1e-3 & FAR=1e-2 && FAR=1e-4 & FAR=1e-3 & FAR=1e-2 \\
      \hline
      Direct  & 94.33 & 96.23 & 97.53  && 95.81 & 97.26 & 98.27\\
      \hline
      \citet{boddeti2018secure}  & 94.27 (-0.06) & 96.20 (-0.03) & 97.52 (-0.01)  &&  \multicolumn{3}{c}{-----------------\textbf{NA}----------------}\\
      \citet{dusmanu2021privacy}  & 0.00{ (-94.33)} & 0.65{ (-95.58)} & 88.70{ (-8.83)} &&  0.00{ (-95.81)} & 9.73{ (-87.53)} & 91.67{ (-6.60)} \\
      \hline
      SecureVector-512  & 94.33{ (\textbf{-0.00})} & 96.23{ (\textbf{-0.00})} & 97.53{ (\textbf{-0.00})} && 95.81{ (\textbf{-0.00})} & 97.26{ (\textbf{-0.00})} & 98.27{ (\textbf{-0.00})}  \\
      \hline
    \end{tabular}
  \end{table*}
}

\noindent
\textbf{Results on LFW/CFP-FP/AgeDB.}
The results are shown in \cref{tab:main}.
IronMask~\cite{kim2021ironmask} can provide a security level of 115.
However, it can only compare a pair of a raw feature and a protected feature, 
and therefore cannot provide full protection to all features.
Besides, the verification accuracies acutely collapse on these benchmarks 
(\ie, -40.41\% on AgeDB).
In addition, it consumes massive resources as high-dimensional matrix operations are involved.
Specifically, enrolling one feature requires 557.92ms and 2040.71Kb of memory, and storing an enrolled feature requires 2100Kb of space.
\citet{boddeti2018secure} utilizes a fully homomorphic encryption (FHE) scheme called BFV~\cite{fan2012somewhat} and speeds up the process by a batching technique.
Its security bit is 128 and performance degradations are small (\ie, -0.00\%, -0.07\% and -0.05\% for LFW/CFP-FP/AgeDB).
However, FHE still suffers from high computational costs, especially when matching two enrolled records.
For example, the average running time is 23.6ms, over 50 times slower than other methods.
Note that matching occurs much more frequently than enrollment (\eg, if we have 100/100 features in the probe/gallery sets, there will be 200 enrollments but 10K comparisons).
The slow running speed limits its usage in large-scale applications.
\citet{dusmanu2021privacy} has no theoretical privacy guarantee, so we put a "NA" here.
The approach runs the most efficiently among all baselines because only basic linear algebra operations are involved.
However, its verification accuracies drops are still too high (\ie, -9.89\% for CFP-FP and -9.05\% for AgeDB) for deployments in sensitive scenes such as biometric authorization.

Our proposed SecureVector-512 can protect all involved features with the highest security bit of 229, which almost doubles the best value of the baselines.
Besides, the verification accuracies are consistent with those when raw features are used directly.
More essentially, our method runs more efficiently and consumes much fewer memory/storage resources than the baselines.
To be more specific, the average running time is 0.59ms for enrollment.
This value is nearly half of the best value of 1.01ms among baselines.
The average time for matching a pair of sanitized features is 0.30ms, which is still the second-best.
For the consumption in memory and storage, our methods overwhelm the corresponding values of all baselines.
These results fully demonstrate the advantages of SecureVector in security level, performance retainability, running speed, and resource consumption.

\noindent
\textbf{Results on IJB-B/IJB-C.}
IJB-B~\cite{whitelam2017iarpa} and IJB-C~\cite{maze2018iarpa} are  top challenging benchmarks for face recognition.
The evaluation on IJB-B requires enrolling 12K features and matching 8M pairs.
For IJB-C, these numbers increase up to 23K and 16M respectively.
We present the true acceptance rates (TARs) when using different false acceptance rates (FARs) in \cref{exp:tab_ijb}.
IronMask~\cite{kim2021ironmask} is not listed as it only produces binary matching scores  instead of contiguous values, which makes the TAR/FAR meaningless in the benchmarks.
We also skip \citet{boddeti2018secure} on IJB-C because the evaluation takes too long to finish.
The reason is that \citet{boddeti2018secure} suffers from low matching speed caused by the FHE scheme, and this low efficiency issue is further aggravated by the tremendous pairs (16M on IJB-C).
For results on IJB-B, \citet{boddeti2018secure} shows high persistence in verification accuracies with only -0.06\% on TAR when the FAR  is 1e-4.
Although achieving acceptable results on LFW/CFP-FP/AgeDB, \citet{dusmanu2021privacy} performs poorly on the challenging benchmarks.
The TAR even collapses into zero when FAR is 1e-4, which denies the applicability on utility-required and large-scale scenarios.
In contrast, the SecureVector-512 has zero performance degradation while is still carried out normally.

\section{Conclusions and Limitations}
\label{sec:conclusion}

In this paper, we propose a plug-in module called SecureVector to prevent privacy leakage in image retrieval systems.
By the carefully designed schemes of random permutations and 4L-DEC converting, SecureVector successfully achieves privacy-preserving, real-time, and lossless feature matching.
Extensive experiments demonstrate the superiority of our approach over the state-of-the-arts on matching accuracies, resource consumption, and security level.

SecureVector mainly has the following limitations:
Firstly, the security of cryptographic methods~\cite{troncoso2013fully, boddeti2018secure, morampudi2020privacy, engelsma2020hers} are with the assumption that the private key is confidential, and SecureVector is no exception.
Secondly, we still cannot achieve comparable speed as direct matching raw features.
Future works, including designing new schemes, are expected to increase the security or speed up the algorithm.

\clearpage

{\small
  \bibliographystyle{plainnat}
  \bibliography{egbib}
}

\appendix

\clearpage
\appendix
\section{Appendix}
\subsection{Mathematical Proofs}

In this section, we provide mathematical proofs for SecureVector.

\subsubsection{Bounds of $W_x$}\label{sec:ext_bound}
\begin{lemma}
  \label{lemma:wf}
  The magnitude $W_x$ of permuted embedding $\bm b_x$ is bounded in the range $[e^{-L/M}, e^{(L-1)/M}]$.
\end{lemma}
\begin{proof}
  The raw embedding $\bm x$ is normalized, \textit{i.e.}, $\|\bm x\|^2_2= (x_1^2+x_2^2+\cdots+x_d^2) = 1$.
  As $u_x^i \in \{0, 1, \cdots, 2L-1\}$ and $s_x^i \in \{-1, 1\}$ for $i=1,2,\cdots K$, it's easily to conclude that  $-L \leq (u_x^i-L) \leq L-1 $ and $(s_x^i)^2 = 1$.
  Then we can have
  \begin{equation*}
    \footnotesizea
    \begin{split}
      W_x^2 &= \|\bm b_x\|^2_2\\
      &= \|(s_x^1 e^{(u_x^1-L)/M}  \bar x_1, s_x^2 e^{(u_x^2-L)/M}  \bar x_2, \cdots, s_x^K  e^{(u_x^K-L)/M}  \bar x_K )\|^2_2 \\
      &= (e^{2(u_x^1-L)/M}x_1^2 + e^{2(u_x^1-L)/M}x_2^2 +  \cdots + e^{2(u_x^1-L)/M}x_{\bar d}^2) + \\
      & \cdots (e^{2(u_x^K-L)/M}x_{(K-1)\bar d + 1}^2 +  +  \cdots + e^{2(u_x^K-L)/M}x_{d}^2) \\
      & \geq  e^{-2L/M} \cdot(x_1^2+x_2^2+\cdots+x_d^2) =  e^{-2L/M}
    \end{split}
  \end{equation*}
  and
  \begin{equation*}
    \footnotesizea
    \begin{split}
      W_x^2 &= \|\bm b_x\|^2_2\\
      &= (e^{2(u_x^1-L)/M}x_1^2 + e^{2(u_x^1-L)/M}x_2^2 +  \cdots + e^{2(u_x^1-L)/M}x_{\bar d}^2) + \\
      &  \cdots (e^{2(u_x^K-L)/M}x_{(K-1)\bar d + 1}^2 +  +  \cdots + e^{2(u_x^K-L)/M}x_{d}^2) \\
      & \leq e^{2(L-1)/M} \cdot(x_1^2+x_2^2+\cdots+x_d^2) =  e^{2(L-1)/M}
    \end{split}
  \end{equation*}
  Therefore, $W_x \in [e^{-L/M}, e^{(L-1)/M}]$.
\end{proof}

\subsubsection{Additional Privacy Analysis}

The adversary may also attempt to recover the original features by the intermediate matching results.
That is also infeasible because the number of equations is much less than that of variables as indicated by \cref{lemma:attack}.
Moreover, all involved variables are integers and form a solution space with large volumes.
It's well-known that integer problems are usually NP-hard to solve.
That also highly increases the hardness for successful attacks.

\begin{lemma}
  \label{lemma:attack}
  An adversary cannot recover the features from the matching results.
\end{lemma}
\begin{proof}
  If an adversary successfully gets access to the intermediate results during matching, \ie, $\{\bm c_x, \tilde T_x\}$, $\{\bm c_y, \tilde T_y\}$, $\{u_z^i, s_z^i\}_{i=1}^K$ and $w_z$.
  Then the adversary can have the following equations:
  \begin{equation}
    \small
    \label{eq:tmp}
    \begin{aligned}
      & s_z^i = s_x^i \cdot s_y^i,  \quad  \forall i=1,2,\cdots, K \\
      & u_z^i = u_x^i  +  u_y^i, \quad \forall i=1,2,\cdots, K \\
      & e^{\frac{w_z - 32L^3}{16L^2M}} = W_x\cdot W_y \\
      & \text{\ where } W_x = \|(s_x^i\cdot e^{(u_x^i-L)/M}\bar {\bm x}_i, \cdots, s_x^K\cdot e^{(u_K^i-L)/M}\bar {\bm x}_K\|_2,\\
      & \qquad W_y = \|(s_y^i\cdot e^{(u_y^i-L)/M}\bar {\bm y}_i, \cdots, s_y^K\cdot e^{(u_K^i-L)/M}\bar {\bm y}_K\|_2.\\
    \end{aligned}
  \end{equation}
  However, there are $2K+1$ equations but $4K$ variables in total.
  That means the original features are impossible to be recovered, especially when $K$ is large.
\end{proof}

\subsubsection{Feature Matching for Unnormalized Features with Different Metrics}
\label{sec:ext_match}

In the main text, we assume that all features are normalized and the metric is the cosine distance.
This assumption may fail in some particular scenarios.
However, a few modifications can accommodate SecureVector to these cases.

\noindent
\textbf{Features are normalized and Euclidean distance is used.}
Note that we can compute the $\bm x^T\bm y$ by \cref{lemma:score} in the maintext.
Then the Euclidean distance of $\bm x, \bm y$ can be easily computed by
\begin{equation}
  \small
  \label{eq:m_c1}
  \begin{split}
    \|\bm x - \bm y\|_2^2 = \bm x^T\bm x + \bm y^T\bm y - 2\bm x^T\bm y = 2 - 2\bm x^T\bm y
  \end{split}
\end{equation}

\noindent
\textbf{Features are unnormalized and dot production is used.}
We only need to replace \cref{eq:b2c} by
\begin{equation}
  \label{eq:m_c2}
  \small
  \begin{aligned}
    \bm c_x =\bm b_x/ W_x,\text{ where }  W_x = \|\bm b_x\|_2/\|\bm x\|_2.
  \end{aligned}
\end{equation}
Then \cref{lemma:score} in the main text can directly produce the dot production $\bm x^T\bm y$ with other things unchanged.

\noindent
\textbf{Features are unnormalized and Euclidean distance is used.}
In this scenario, we also need to replace \cref{eq:b2c} by \cref{eq:m_c2} and compute the dot production $\bm x^T\bm y$ first.
Note that
\begin{equation}
  \small
  \begin{aligned}
    \|\bm c_x\|_2 = \|\bm b_x / W_x\|_2 = \|\bm b_x\|_2 \frac{\|\bm x\|_2}{\|\bm b_x\|_2} = \|\bm x\|_2.
  \end{aligned}
\end{equation}
Then we can calculate  the Euclidean distance by
\begin{equation}
  \label{eq:m_c3}
  \small
  \begin{aligned}
    \|\bm x - \bm y\|_2^2 & = \bm x^T\bm x + \bm y^T\bm y - 2\bm x^T\bm y \\
    & = \|\bm c_x\|_2^2 + \|\bm c_y\|_2^2  - 2\bm x^T\bm y.
  \end{aligned}
\end{equation}

\subsection{Additional Experiments}
\label{sec:ext_exp}

\subsubsection{Comparisons of Paillier and FHE Schemes}\label{sec:addexp_he}
To verify the advantages of Paillier over current homomorphic encryption schemes such as BFV~\cite{fan2012somewhat} and CKKS~\cite{cheon2017homomorphic}, we utilize two open-source packages: phe\footnote{https://github.com/data61/python-paillier.} for Paillier and TenSEAL\footnote{https://github.com/OpenMined/TenSEAL.} for BFV and CKKS.
We record the time for encrypting plaintexts, adding and decrypting ciphertexts on 10K randomly generated integers.
\cref{tab:compare_he} presents the results.
BFV and CKKS can have 128 or 192 bits of security in practice.
When reaching 128 bits of security, BFV requires about $485\times$ memories and $50\times, 52\times, 82\times$ time for three operations compared to Paillier.
For CKKS, the numbers are around $727\times$ and $86\times, 69\times, 122\times$.
In addition to the speed advantages, Paillier can also easily reach high bits of security (\eg, 512 and 1024) while requires much less computational costs than BFV and CKKS.
We should note here that phe is implemented by python while the core codes in TenSEAL are written with C++.
The advantages of Paillier will be enlarged if the gap is bridged.

{
  \setlength{\tabcolsep}{3.6pt}
  \begin{table}[htb!]
    \centering
    \caption{Comparisons of homomorphic encryption schemes during encrypting, adding and decrypting integers.
      Here we set the dimension in BFV and CKKS  to be the common-used 4096.
    }
    \footnotesizea
    \begin{tabular}{cccccr}
      \hline
      Method & Bits of  & \multicolumn{3}{c}{Avg. Time (ms)} & Memory \\
      \cline{3-5}
             & Security  & Encrypt & Add& Decrypt & \multicolumn{1}{c}{(Mb)} \\
      \hline
      BFV~\cite{fan2012somewhat} & 128  & 1.962 & 0.211 & 0.739    & 3883  \\
      CKKS~\cite{cheon2017homomorphic}  & 128  & 3.372 & 0.278 & 1.103    & 5822  \\
      \hline
      \multirow{4}{*}{Paillier~\cite{paillier1999public}}
             & 128  & 0.039  & 0.004  & 0.009    & 8  \\
             & 256  & 0.087  & 0.006  & 0.025    & 9  \\
             & 512  & 0.335  & 0.010  & 0.100    & 11  \\
             & 1024 & 2.120  & 0.023  & 0.574    & 15 \\
      \hline
    \end{tabular}
    \label{tab:compare_he}
  \end{table}
}

\subsubsection{Experiments on Parameter $M$}
\label{sec:apexp_param}

To pick a proper value of $M$, we use pre-trained models\footnote{https://github.com/IrvingMeng/MagFace.} provided by authors of MagFace~\cite{meng2021magface} to extract 1K features from MS1Mv2~\cite{deng2019arcface}.
These features are of dimension 512, and we set $K=64$.
We permute each feature with $\{u_x^i, s_x^i\}_{i=1}^K$ under different $L/M$, and calculate the cosine similarities between the original and the permuted features.
\cref{fig:param_m} shows the results with each feature permuted 100 times.
If only using $\{u_x^i\}_{i=1}^K$, the cosine similarities decrease if we increase $L/M$.
$L/M$ larger than 4 seems enough as the similarities become mostly less than 0.6, which can be regarded as ``dis-similar'' in face recognition.
If only using $\{s_x^i\}_{i=1}^K$, $\bm x$ and $\bm c_x$ are highly dis-similar.
With both $\{u_x^i\}_{i=1}^K$ and $\{s_x^i\}_{i=1}^K$ used, the similarity distributions are in complex patterns, which further benefits privacy protection.
We choose $L/M=128$, in the end, to be on the safe side.

\begin{figure}[htb!]
  \centering
  \subfloat[\label{fig:m1}]
  {\includegraphics[width=0.25\textwidth]{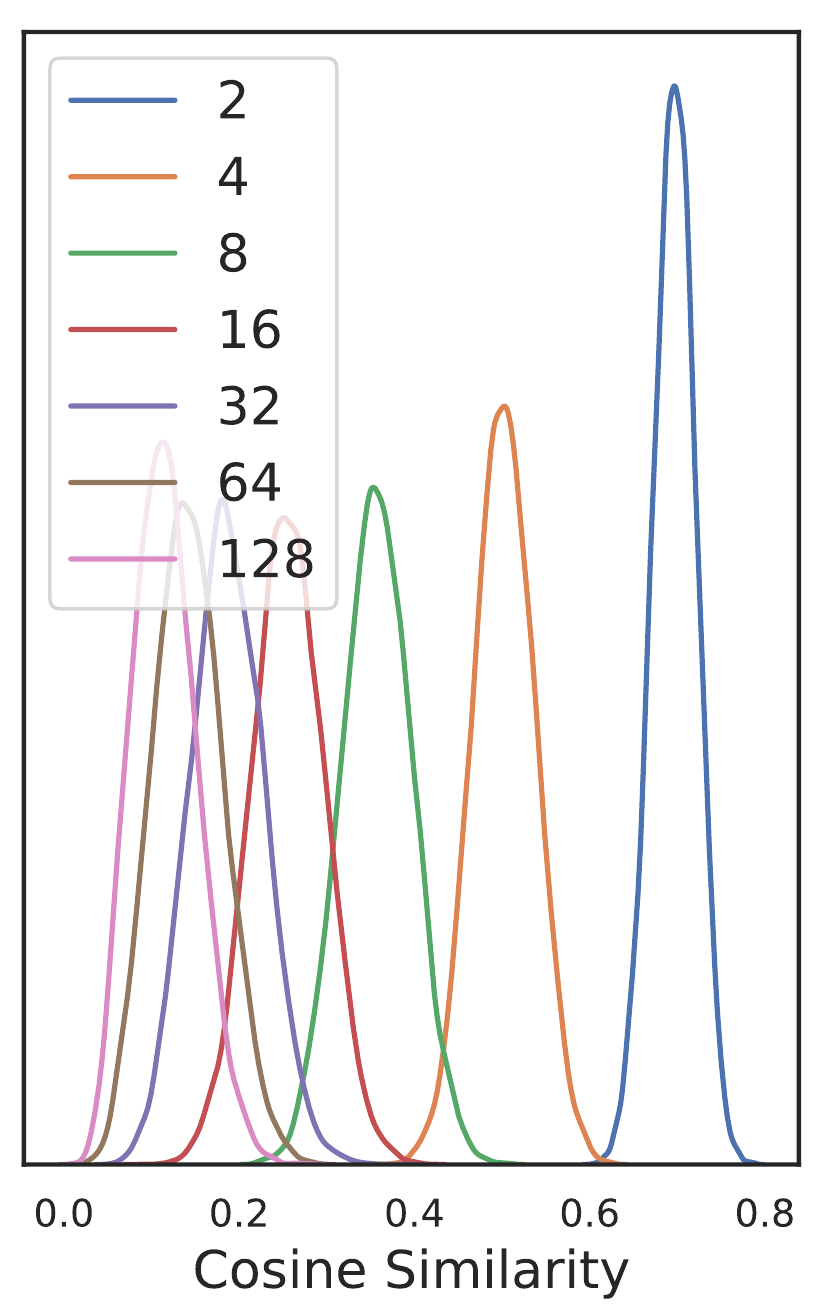}}
  \subfloat[\label{fig:m2}]
  {\includegraphics[width=0.25\textwidth]{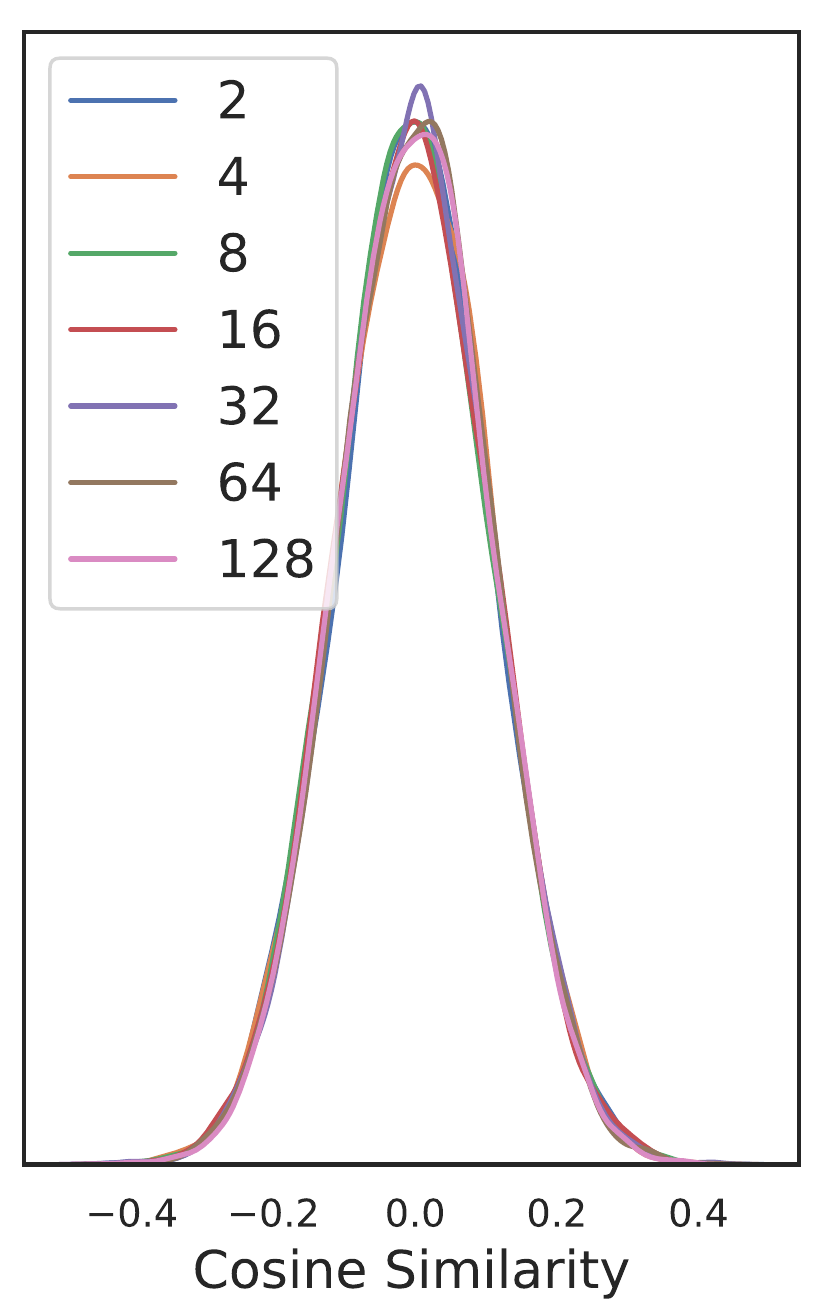}}
  \subfloat[\label{fig:m3}]
  {\includegraphics[width=0.25\textwidth]{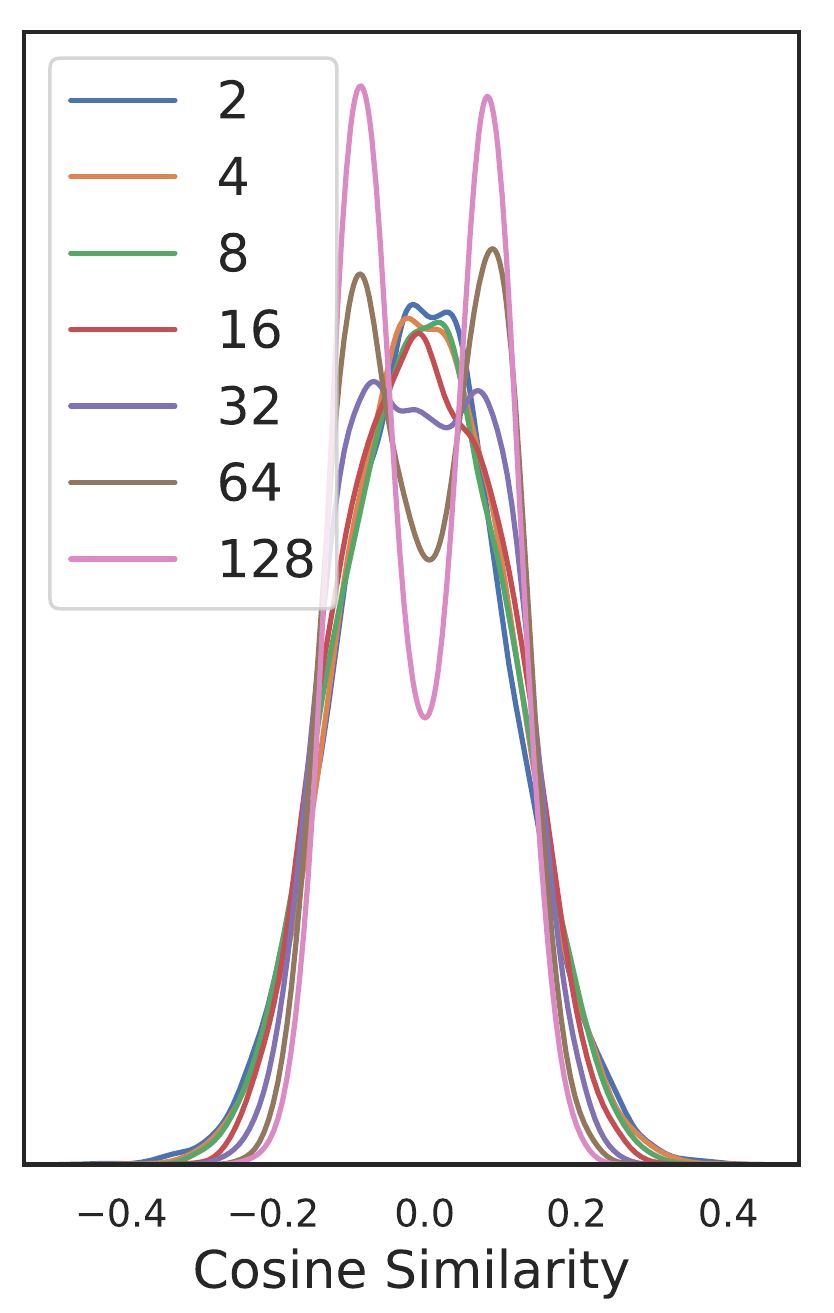}}
  \caption{ Cosine similarities of $\bm x$ and $\bm c_x$ with different $L/M$.
    (a) Permuted with $\{u_x^i\}_{i=1}^K$ only.
    (b) Permuted with $\{s_x^i\}_{i=1}^K$ only.
    (c) Permuted with both $\{u_x^i\}_{i=1}^K$ and $\{s_x^i\}_{i=1}^K$.
  } \label{fig:param_m}
\end{figure}

\subsubsection{Experimental Privacy Analysis}
\label{sec:exp_pa}

\cref{sec:method_an} has already theoretically proved the security of SecureVector.
This part shows that our method is resistant to known experimental attacks.
With exposed features given, two attacks are taken into account: inferring face images and inferring attributes (\eg, gender and race).


\begin{figure}[htb!]
  \centering
  \includegraphics[width=0.7\textwidth]{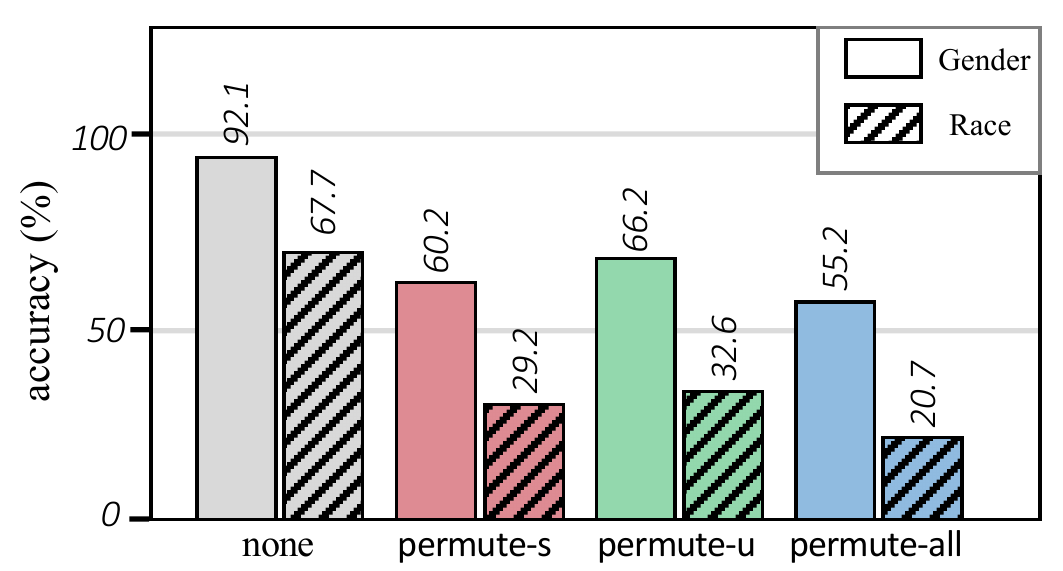}
  \caption{
    Accuracies of predicted genders and races on FairFace.
    ``none'' means results are predicted by raw features.
    ``permute-s'', ``permute-u''  and ``permute-all'' means features are permuted with sign flipping ($s_x^i$) , value scaling ($u_x^i$) and both, respectively.
  } \label{fig:vis_gr}
\end{figure}

\noindent
\textbf{Inferring distinguishable attributes.}
We extract features from 97K images (86K for training and 11K for testing) on FairFace~\cite{karkkainen2021fairface}, a dataset with balanced races (7 groups) and genders (2 groups).
Genders and races of a test image are predicted by its 10-nearest neighbors over the training images and the results are shown in \cref{fig:vis_gr}.
Without permutation, an adversary can infer genders and races with accuracies of 92.1\% and 67.7\%, which results in privacy threats.
Our strategy of sign flipping or value scaling on features can significantly relieve the threats.
If using both permutations, the accuracies are 55.2\% and 20.7\% respectively, which are very close to random guesses (\ie, 50\% and 14\%).

\noindent
\textbf{Inferring face images.}
We first sample 1K classes from MS1Mv2~\cite{deng2019arcface} dataset and extract features with the model we used in the main text.
Then, one sample from each class is chosen as the attacking target.
Following \citet{boddeti2018secure}, we consider two types of attacks in this scenario.
The first is a direct inversion attack, where features are projected into the corresponding face images with a pre-trained generative model.
We design a reconstruction network (modified from DCGAN~\cite{radford2015unsupervised}) 
which projects a 512 dimensional embedding to a face image.
We train the generative model on 10K classes from MS1Mv2~\cite{deng2019arcface}, which have no overlaps with the 1K classes for testing.
With the supervision of features and images, we train the network by Adam optimizer with L1 loss, which penalizes the absolute value of pixel differences between reconstructed and real images.
The learning rate is initialized from $0.001$ and divided by 10 every 10 epochs, and we stop the training at the 50th epoch.
The weight decay is $5e-4$ and batch size is $512$ during the training.

The second one is the nearest neighbor attack, which requires an adversary to find the top similar images from an existing face database.
The database is simulated by the remaining images/features of previously sampled 1K classes.
\cref{fig:vis_attack_ext} show our results.

\begin{figure*}[htb!]
  \centering
  \includegraphics[width=0.85\textwidth]{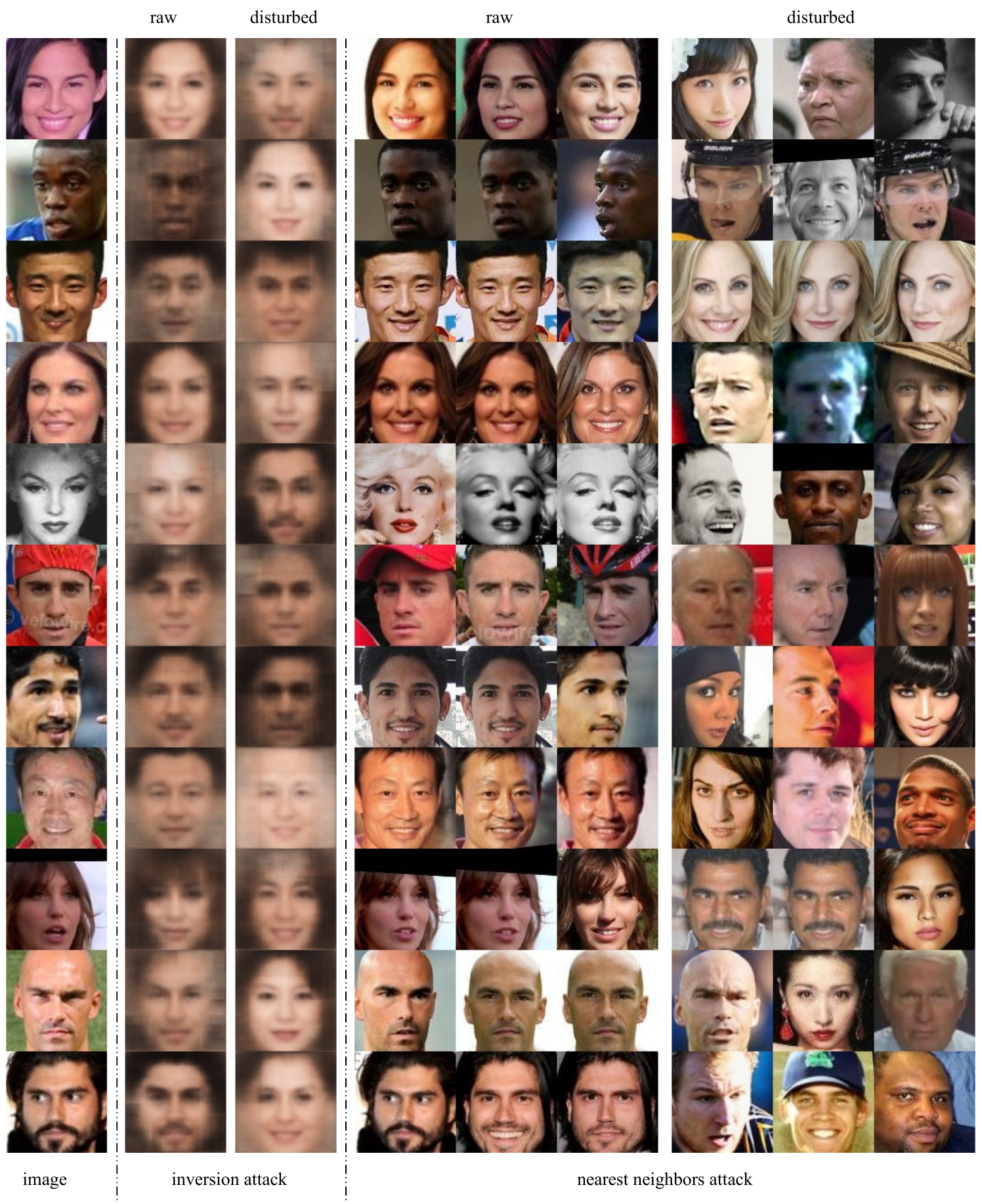}
  \caption{
    Several qualitative examples for privacy analysis:
    first the source images, then reconstructed images by the raw features and permuted probe features, and in the end the top-3 nearest faces with features before and after permutations.
  }\label{fig:vis_attack_ext}
\end{figure*}

If using the raw features, the reconstructed face images by the generative model look similar to the real faces in human eyes.
Besides,  one can easily find face images from the same identities with the nearest neighbor attack.
These all indicate that the raw features are vulnerable to malicious attacks.
In contrast, our permutation scheme can dramatically protect privacy as the reconstructed faces differ greatly from the original images and the nearest neighbors can no longer reveal correct identities.
This demonstrates the efficacy of the proposed random permutation scheme in SecureVector.

Our privacy-protection mechanism can be partially explained by \cref{fig:vis_feat}, where T-SNE visualizes features from 10 classes.
After permutations, features from different classes are no longer discriminative, which prevents an adversary from inspecting the privacy.
\begin{figure}[htb!]
  \centering
  \subfloat[Raw features.]{\includegraphics[trim={60pt 70pt 60pt 70pt},clip, width=0.4\textwidth]{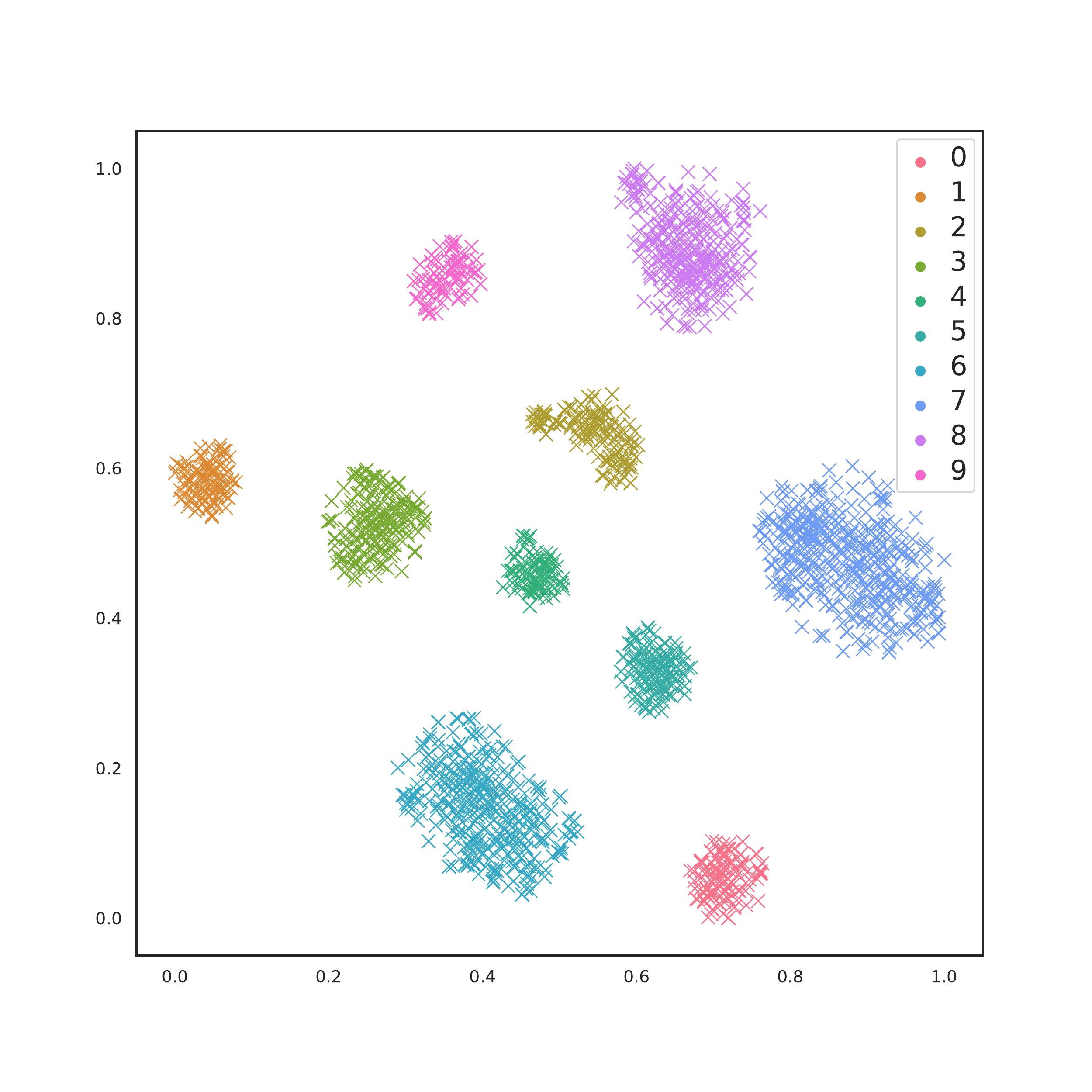}}
  \subfloat[Permuted features.]{\includegraphics[trim={60pt 70pt 60pt 70pt},clip,width=0.4\textwidth]{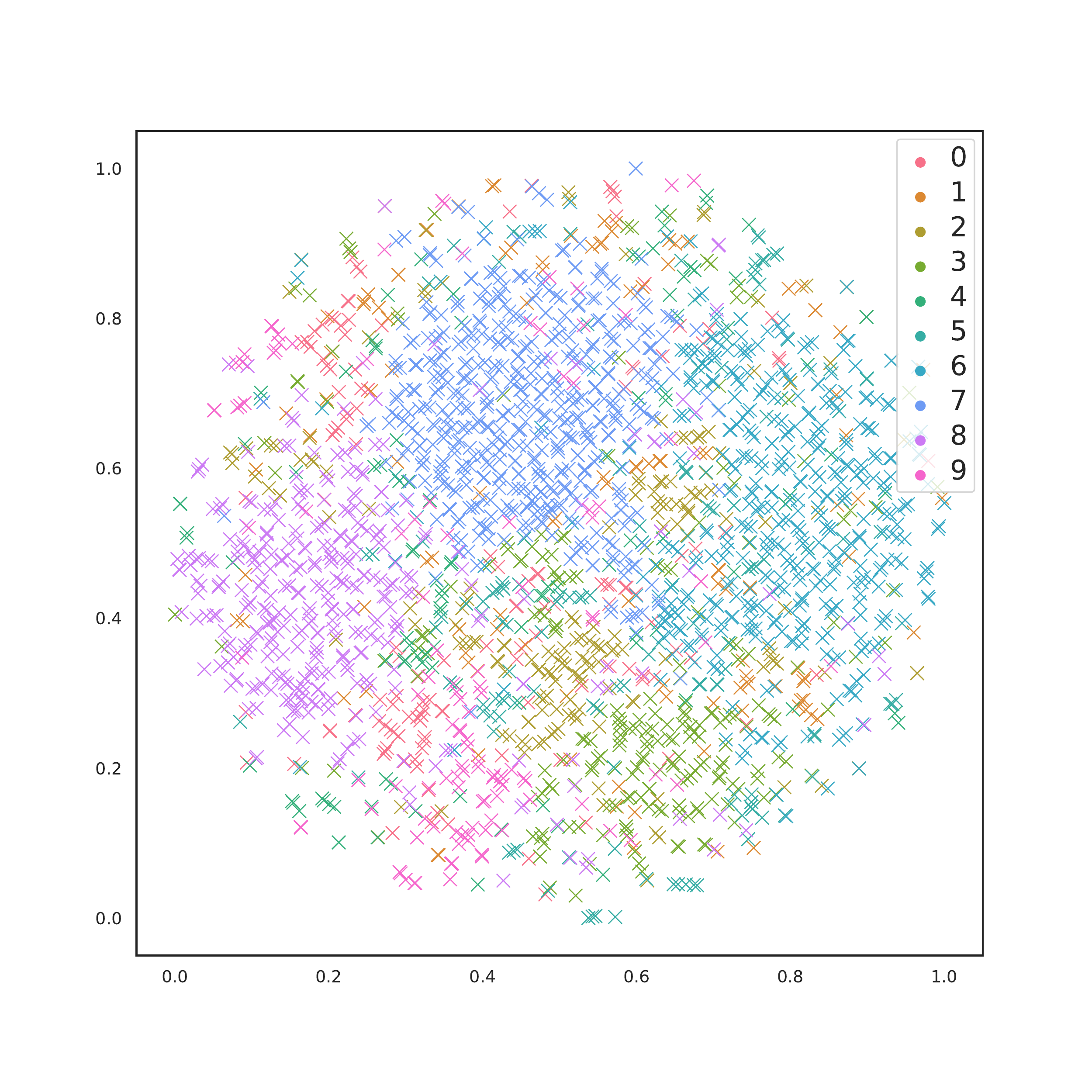}}
  \caption{Visualization of 10 classes from MS1Mv2 by T-SNE.
  }
  \label{fig:vis_feat}
\end{figure}{}

\subsubsection{Additional Experiments on Face Recognition}
\label{sec:exp_ext_face}
We further evaluate these methods with three additional  recognition models $\bm \phi_1, \bm \phi_2, \bm \phi_3$ which are trained with different settings (details can be found in \cref{tab:fr_details}).
\cref{tab:fr_om} presents the results, from where we can see that the baselines more or less encounter performance drops while our SecureVector does not suffer from that issue..

{
  \begin{table}[htb!]
    \centering
    \caption{Details of feature extractors.}
    \label{tab:fr_details}
    \footnotesizea
    \begin{tabular}{cccc}
      \hline
      Model & Loss & Backbone & Training Dataset \\
      \hline
      $\bm \phi_1$ & MagFace~\cite{meng2021magface} & ResNet50 & MS1Mv2~\cite{deng2019arcface} 
                                                                 
      \\
      $\bm \phi_2$ & MagFace~\cite{meng2021magface} & ResNet18 & CASIA-WebFace ~\cite{yi2014learning} \\
      $\bm \phi_3$ & ArcFace~\cite{deng2019arcface} & ResNet18 & CASIA-WebFace~\cite{yi2014learning} \\
      \hline
    \end{tabular}
  \end{table}
}

{
  \setlength{\tabcolsep}{2pt}
  \begin{table}[htb!]
    \centering
    \caption{Performances of plug-in feature protection methods with different recognition models.}
    \label{tab:fr_om}
    \footnotesizea
    \begin{tabular}{cclll}
      \hline
      \multirow{2}{*}{Model} & \multirow{2}{*}{Method} &\multicolumn{3}{c}{Verification Accuracy (\%)}\\
      \cline{3-5}
                             & &LFW& CFP-FP & AgeDB\\
      \hline
      \multirow{5}{*}{$\bm \phi_1$} & Direct & 99.78 & 97.64 & 98.00\\
      \cline{2-5}
                             & IronMask~\cite{kim2021ironmask} & 86.23 (-13.55) & 54.63 (-43.01) & 55.38 (-42.62) \\
                             & \citet{boddeti2018secure} & 99.77 (-0.01) & 97.63 (-0.01) & 97.95 (-0.05) \\
                             & \citet{dusmanu2021privacy} & 97.85 (-1.93) & 85.44 (-12.2) & 88.07 (-9.93)\\
                             & SecureVector-512 & 99.78 (\textbf{-0.00}) & 97.64 (\textbf{-0.00}) & 98.00 (\textbf{-0.00})\\
      \hline
      \multirow{5}{*}{$\bm \phi_2$} & Direct & 99.20 & 93.23 & 93.20\\
      \cline{2-5}
                             & IronMask~\cite{kim2021ironmask} & 70.60 (-28.60) & 50.90 (-42.33) & 51.33 (-41.87) \\
                             & \citet{boddeti2018secure} &99.18 (-0.02) & 93.14 (-0.09) & 93.07 (-0.13) \\
                             & \citet{dusmanu2021privacy} &91.57 (-7.63) & 71.84 (-21.39) & 73.68 (-19.52) \\
                             & SecureVector-512 & 99.20 (\textbf{-0.00}) & 93.23 (\textbf{-0.00})& 93.20 (\textbf{-0.00})\\
      \hline
      \multirow{5}{*}{$\bm \phi_3$} & Direct & 99.03  & 91.97 & 93.02\\
      \cline{2-5}
                             & IronMask~\cite{kim2021ironmask} & 70.83 (-28.20) & 51.19 (-40.78) & 51.28 (-41.74) \\
                             & \citet{boddeti2018secure} & 99.07 (\textbf{+0.04}) & 91.93 (-0.04) & 92.98 (-0.04) \\
                             & \citet{dusmanu2021privacy} & 90.63 (-8.40) & 71.87 (-20.10) & 74.52 (-18.50)\\
                             & SecureVector-512 & 99.03 (-0.00) & 91.97 (\textbf{-0.00})& 93.02 (\textbf{-0.00})\\
      \hline
    \end{tabular}
  \end{table}
}

\subsubsection{Person Re-Identification}
\label{sec:exp_reid}

We adopt two public models\footnote{https://github.com/guxinqian/Simple-ReID.} with feature dimension 2048 as evaluators on the person re-identification task.
Market-1501~\cite{Zheng2015Scalable} is employed as the benchmark, where 3K images are queried over a gallery of 16K images.
Note that \citet{boddeti2018secure} and IronMask~\cite{kim2021ironmask} are not evaluated on the benchmark as their evaluations are computational heavy with 19K features to enroll and 53.6M pairs to match.
\cref{tab:exp_reid} shows our results, which demonstrate that \citet{dusmanu2021privacy} still suffers from severe performance drops while SecureVector achieves lossless accuracies on the task.

{
  \setlength{\tabcolsep}{3.5pt}
  \begin{table}[htb!]
    \centering
    \caption{Performances on the Market-1501 benchmark.}
    \label{tab:exp_reid}
    \footnotesizea
    \begin{tabular}{clll}
      \hline
      &  Method & mAP (\%)  & Top-1 (\%) \\
      \hline
      \multirow{3}{*}{Model 1}
      & Direct & 84.75 & 93.56 \\
      \cline{2-4}
      & \citet{dusmanu2021privacy} & 5.52 (-79.23) & 0.06 (-93.50) \\
      & SecureVector-512 & 84.75 (-0.00) & 93.56 (-0.00) \\
      \hline
      \multirow{3}{*}{Model 2}
      & Direct & 87.43 & 94.95 \\
      \cline{2-4}
      & \citet{dusmanu2021privacy} & 6.22 (-81.21) & 0.18 (-94.77)  \\
      & SecureVector-512 & 87.43 (-0.00)& 94.95 (-0.00)\\
      \hline
    \end{tabular}
  \end{table}
}

In \cref{fig:vis_ext_reid}, we present some visualizations for top-3 matching results when using the raw features, \citet{dusmanu2021privacy}  and SecureVector.
\citet{dusmanu2021privacy} cannot work properly as the top-3 results are incorrect in most cases.
In contrast, SecureVector can get exactly the same results as when using the raw features, which demonstrates the efficiency of our method.

\begin{figure*}[htb!]
  \centering
  \includegraphics[width=0.8\textwidth]{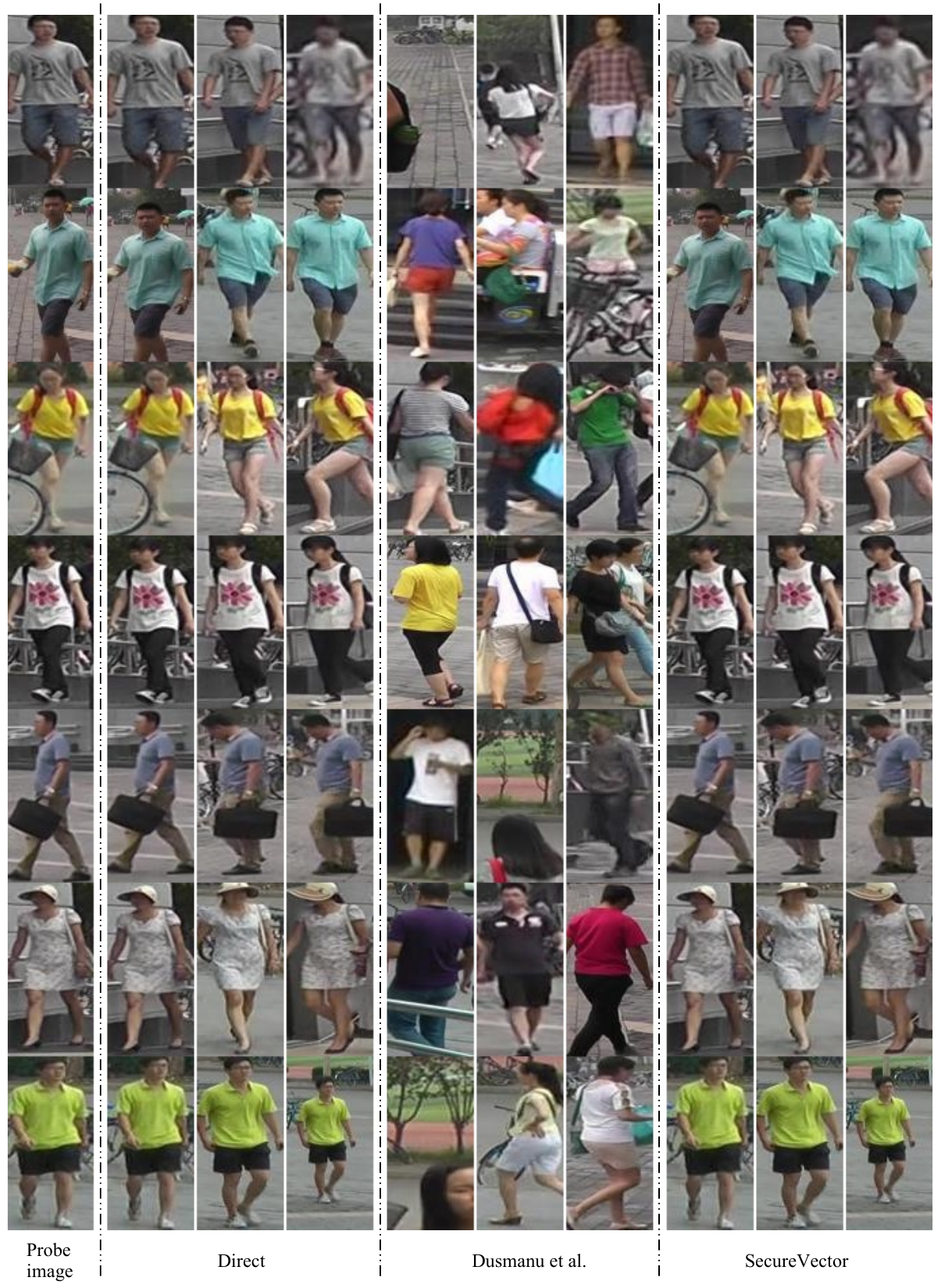}
  \caption{
    Top-3 matching results with different methods on the Market-1501 benchmark when using raw features, \citet{dusmanu2021privacy} and SecureVector.
  }\label{fig:vis_ext_reid}
\end{figure*}

\subsubsection{Image Retrieval}
\label{sec:exp_ir}

In this section, we evaluate template protection methods on image retrieval with two feature extractors\footnote{https://github.com/filipradenovic/cnnimageretrieval-pytorch} released by authors of \citet{radenovic2018fine}.
We choose two benchmarks including Oxford5k~\cite{philbin2007object} (55 query and 5063 gallery images) and Paris6k~\cite{philbin2008lost} (55 query and 6392 gallery images), and report the results in \cref{tab:exp_ir}.
SecureVector again proves its superiority in performance preservation.

{
  \setlength{\tabcolsep}{3.5pt}
  \begin{table}[htb!]
    \centering
    \caption{mAP (\%) on the image retrieval task.}
    \label{tab:exp_ir}
    \footnotesizea
    \begin{tabular}{lccccc}
      \hline
      \multirow{2}{*}{Method} & \multicolumn{2}{c}{Model 1} && \multicolumn{2}{c}{Model 2} \\
      \cline{2-3}
      \cline{5-6}
                              & Oxford5K & Paris6K && Oxford5K & Paris6K \\
      \hline
      Direct & 81.13 & 87.80 && 87.80 & 92.59 \\
      \hline
      IronMask~\cite{kim2021ironmask} & 2.62 & 6.35 && 32.32 & 34.71\\
      \citet{boddeti2018secure} & 9.33 & 2.16 && 11.23 & 19.65\\
      \citet{dusmanu2021privacy} & 81.10 & 87.83 && 87.80 & 92.48\\
      \hline
      SecureVector-512 & 81.13 & 87.80 && 87.80 & 92.59 \\
      \hline
    \end{tabular}
  \end{table}
}

\cref{fig:vis_ext_ir} is the visualization for the image retrieval task.
We present the top-3 results when using the raw features as well as different template protection methods.
Note that IronMask~\cite{kim2021ironmask} is ignored as it cannot produce similarity scores, which makes the top-3 results meaningless.
From the figure, we can see that \citet{dusmanu2021privacy} cannot match the query to the correct images.
\citet{boddeti2018secure} can generate the correct results in most cases, which however differ from the original ones sometimes.
In contrast, results from SecureVector are always aligned with those when using raw features, which can show the superiority of the proposed method.

\begin{figure*}[htb!]
  \centering
  \includegraphics[width=\textwidth]{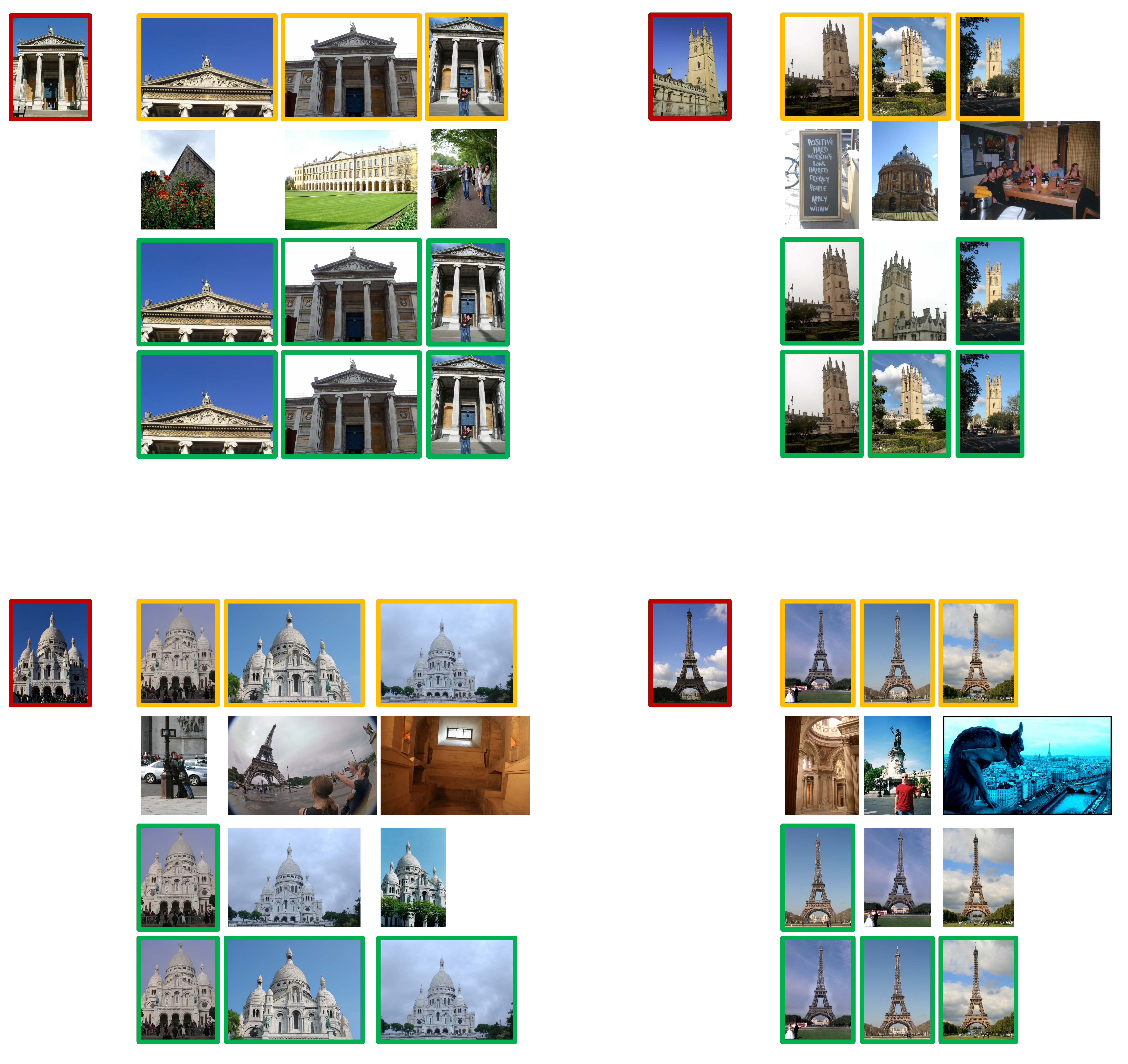}
  \caption{
    Top-3 nearest neighbors from the raw features, \citet{dusmanu2021privacy}, \citet{boddeti2018secure} and SecureVector, which corresponds to 4 rows in each block.
    The top two blocks are on Oxford5K while the bottom two blocks are on Paris6K.
    In each block, the image with the red bounding boxes is the query and images with the yellow bounding boxes are the top-3 similar images when using the raw features.
    Images with the green bounding boxes mean the queried images are the same as the results using the raw features.
    SecureVector can get the most precise results.
  }\label{fig:vis_ext_ir}
\end{figure*}





\end{document}